\documentclass{article}
\usepackage{algorithm}
\usepackage{algorithmic}
\usepackage{wrapfig}
\usepackage{makecell}
\newcommand{\shadecolumn}{\cellcolor{gray!20}}
\newcommand{\shadecolumnlight}{\cellcolor{gray!10}}
\usepackage[table]{xcolor}   % for \cellcolor and \columncolor
\usepackage{array}           % for \newcolumntype
\newcolumntype{s}{>{\columncolor{gray!10}}c}
\newcolumntype{g}{>{\columncolor{gray!20}}l}
\newcommand\independent{\protect\mathpalette{\protect\independenT}{\perp}}
\def\independenT#1#2{\mathrel{\rlap{$#1#2$}\mkern2mu{#1#2}}}

\PassOptionsToPackage{numbers, compress}{natbib}

\usepackage[main, final]{neurips_2026}

\usepackage[utf8]{inputenc} % allow utf-8 input
\usepackage[T1]{fontenc}    % use 8-bit T1 fonts
\usepackage{hyperref}       % hyperlinks
\usepackage{url}            % simple URL typesetting
\usepackage{booktabs}       % professional-quality tables
\usepackage{amsfonts}       % blackboard math symbols
\usepackage{nicefrac}       % compact symbols for 1/2, etc.
\usepackage{microtype}      % microtypography
\usepackage{xcolor}         % colors

\usepackage{graphicx}
\usepackage{subcaption}
\usepackage{amsmath}
\usepackage{amssymb}
\usepackage{mathtools}
\usepackage{amsthm}
\usepackage{bbm}
\usepackage{math_commands}
\usepackage{verbatim}
\usepackage{enumitem}
\newtheorem{theorem}{Theorem}[section]
\newtheorem{lemma}[theorem]{Lemma}

\title{On the Error-Correcting Effects of Stochasticity in Discrete Diffusion}
\author{William Yuan \\
  Georgia Institute of Technology \\
  \texttt{wyuan65@gatech.edu} \\
  \And
  Sungwon Jeong \\
  Georgia Institute of Technology \\
  \texttt{sjeong304@gatech.edu} \\
  \And
  Amirali Aghazadeh \\
  Georgia Institute of Technology \\
  \texttt{aaghazadeh3@gatech.edu} 
}

\begin{document}

\maketitle
\begin{abstract}
    Discrete diffusion models achieve strong performance in text and image generation, but their inference remains slow and must inherently balance sampling efficiency and sample quality. In this work, we present a systematic study of how the \emph{degree of stochasticity} in Markov transitions governs the sampling tradeoff. We show that highly deterministic transitions converge rapidly but suffer from error accumulation, while more stochastic transitions converge more slowly yet can achieve higher final sample quality. Using an information-theoretic analysis, we identify the underlying mechanism as an error-correcting effect induced by \emph{redundant transitions} that symmetrically exchange mass between states, and show that these transitions can provably contract sampling errors. To demonstrate the practical implications of our theoretical findings, we propose \emph{Discrete Churn and Restart Sampling} (DCRS), a sampling procedure that injects controlled stochasticity by alternating between forward and reverse diffusion processes. Experiments on synthetic datasets and large-scale benchmarks show that DCRS improves the speed-quality tradeoff in the low number of function evaluations regime. On image datasets, DCRS achieves up to a $10\times$ reduction in sampling steps compared to standard samplers while maintaining competitive sample quality, whereas on language benchmarks, we observe more nuanced behavior depending on the corruption process and sampling procedure.
\end{abstract}

\begin{figure*}[t]
    \centering
    \includegraphics[width=\linewidth]{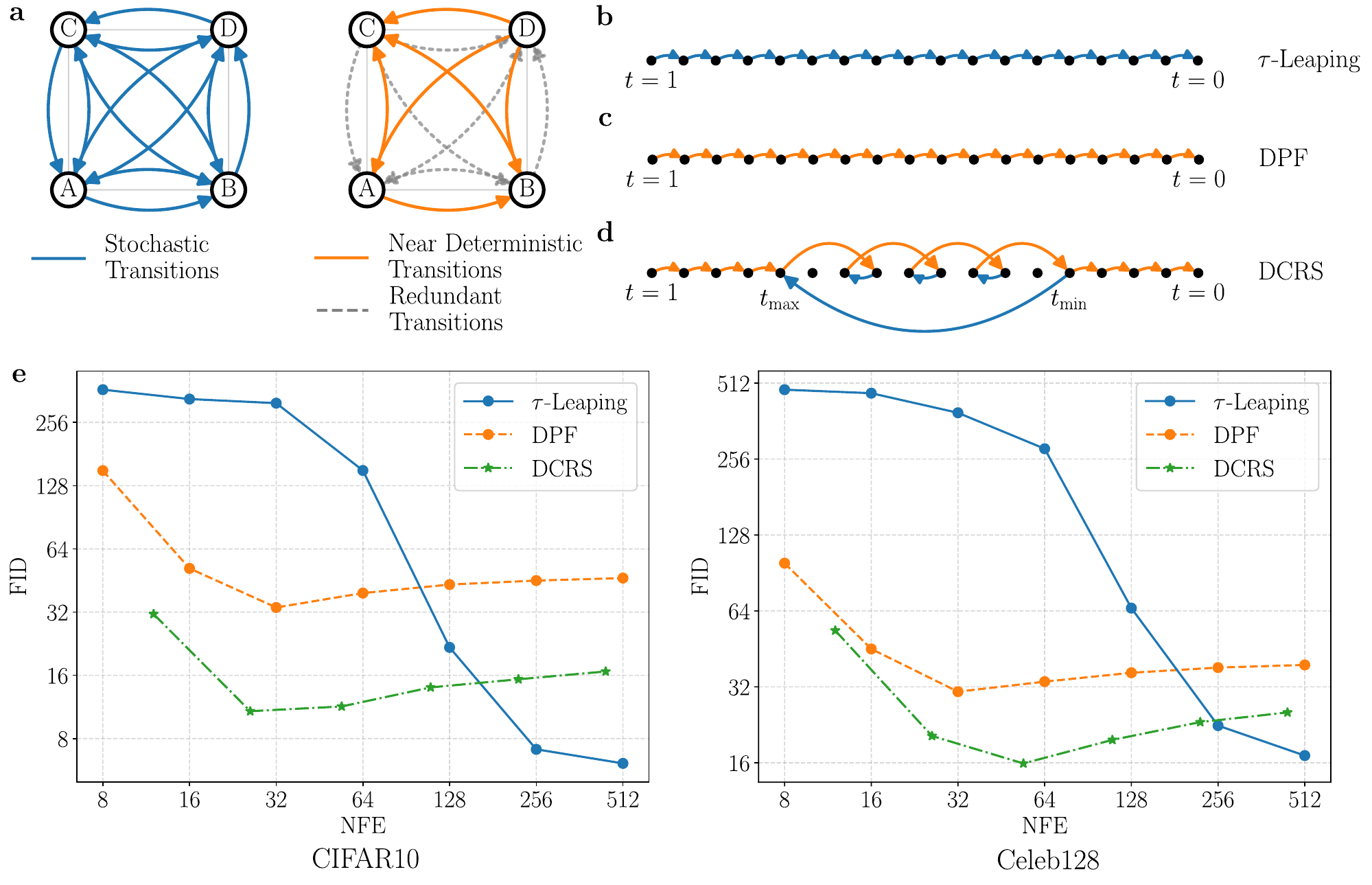}
    \vspace{-2em}
    \caption{{\bf Stochasticity induces error correction in discrete diffusion, enabling an improved speed-quality tradeoff.} \textbf{(a)} Near-deterministic transitions (orange) obtained by removing redundant transitions from the stochastic kernel (blue), resulting in faster but less error-corrective dynamics. \textbf{(b)} Sampling with fully stochastic transitions, where redundant transitions redistribute probability mass between states and help correct accumulated errors. \textbf{(c)} Discrete Probability Flow (DPF), which removes redundant transitions to produce nearly deterministic updates that accelerate convergence but are more sensitive to model errors. \textbf{(d)} Our proposed Discrete Churn and Restart Sampling (DCRS), which alternates between near-deterministic reverse updates and injected forward-process noise. Restarting introduces stochasticity using the forward process without additional neural network evaluations, while churning provides local corrections. \textbf{(e)} Empirical results on CIFAR10 and CelebA illustrating the speed-quality tradeoff: DCRS achieves improved sample quality in the low number of function evaluations (NFE) regime.}
    \label{fig:overview}
    \vspace{-2em}
\end{figure*}
\section{Introduction}
Discrete diffusion models, a class of generative models that iteratively refine samples via discrete-space Markov chains~\citep{MD,D3PM,TLDR}, have achieved competitive performance in natural language~\citep{SEDD,MDLM} and image~\citep{VQDiffusion,Simplified} generation. A central challenge in discrete diffusion is understanding the tradeoff between sample quality and sampling speed. Despite its practical importance, this tradeoff is poorly studied and is widely believed to depend on the degree of stochasticity of the generation process. 

Existing intuitions are largely heuristic. Masking-based transitions are irreversible during sampling and are therefore thought to be susceptible to error accumulation~\citep{SCUD,MaskedCorrection}, while highly stochastic uniform transitions are thought to exhibit self-correcting behavior~\citep{UDLM,Duality}. However, comparisons across transition types are inherently confounded by differences in learned denoising networks. Even under a fixed denoising model, prior work has treated stochasticity as a hyperparameter for improving sample quality, yet its underlying role in error propagation remains poorly understood~\citep{DDIM,ReMDM,DFOT,MMF,DFM}.

In this work, we show that stochasticity plays a fundamental and previously underexplored role in discrete diffusion inference. Specifically, we demonstrate a tradeoff: highly deterministic transitions enable faster convergence but achieve lower sample quality, whereas more stochastic transitions converge more slowly to higher sample quality.  Crucially, we show that this behavior arises from an \emph{error-correcting mechanism} induced by \emph{redundant transitions} between state pairs (Figure~\hyperref[fig:overview]{1a}), which symmetrically exchange probability mass without altering the time marginal distributions. Using an information-theoretic analysis, we show that these transitions can provably contract errors in the KL divergence. This perspective provides a principled explanation for when and why stochasticity improves generation quality, beyond prior heuristic arguments.

These results unify a broader principle. Stochasticity improves generation by mixing away errors, but at the cost of slower sampling. While similar tradeoffs have been observed in continuous diffusion models~\citep{DDIM,SDE,EDM,ODESDE,ProbabilityODEProof}, we show that in the discrete setting, the strength of error correction depends critically on the \emph{type and degree of stochasticity} feasible under the chosen corruption process.

Motivated by this understanding, we propose \emph{Discrete Churn and Restart Sampling} (DCRS), a training-free inference procedure that injects stochasticity in a controlled manner. DCRS alternates between near-deterministic reverse steps and strategically introduced forward-process noise through two mechanisms: (1) \emph{churning}, which performs local forward-reverse perturbations, and (2) \emph{restarting}, which periodically resets the sampling trajectory using forward transitions (Figure~\hyperref[fig:overview]{1d}). Notably, the restart steps introduce stochasticity via the forward process \emph{without additional neural network evaluations}, enabling efficient global error correction across the trajectory.

We evaluate DCRS on discrete diffusion models trained on large-scale image and language datasets. We find that DCRS substantially improves the speed-quality tradeoff in the low number of function evaluations (NFE) regime on image benchmarks, achieving comparable sample quality with significantly fewer sampling steps (Figure~\hyperref[fig:overview]{1e}). On language tasks, we observe more nuanced behavior, indicating that the strength of error-correcting stochasticity may depend on the domain and biases in the sampling procedure.

Our contributions in this work are as follows:
\begin{itemize}[leftmargin=10pt]
\item We provide a systematic characterization of how stochasticity governs the sampling speed-quality tradeoff in discrete diffusion models. Furthermore, we show that stochasticity can correct for \emph{model approximation errors} at the cost of slower sampling, while reducing stochasticity accelerates sampling but weakens error correction.
\item We identify an error-correcting mechanism induced by redundant transitions and show, via an information-theoretic analysis, that adding such transitions can provably contract sampling errors. 
\item We develop DCRS, a sampling procedure that exploits the error-correcting effects of stochasticity by alternating the forward and reverse processes, improving the speed-quality tradeoff in the low-NFE regime.
\end{itemize}
\vspace{-.5em}
\section{Background}
\vspace{-.5em}
\noindent {\bf Continuous-Time Discrete Diffusion Models.} Discrete diffusion models operate on spaces $\gS^D$, where $\gS = \{1,\ldots,S\}$ is a finite state space and $D$ denotes the sequence dimension. These models define a forward Markov process over time $t \in [0,1]$ that evolves a data distribution $p_0$ toward a limiting distribution $p_1$. The time marginals $\{p_t\}_{t\in[0,1]}$ are governed by a time-inhomogeneous continuous-time Markov chain (CTMC) with rate matrix $R_t \in \mathbb{R}^{S^D \times S^D}$ satisfying $\frac{dp_t}{dt} = R_t^Tp_t$. A generative model is obtained by simulating the corresponding reverse-time process starting from $p_1$. 

The reverse rate matrix $\hat{R}_t$ must satisfy the time-reversal relation, $\hat{R}_t(x,y) = \frac{p_t(y)}{p_t(x)} R_t(y,x),$ which ensures the reversed process follows the time marginal distributions $p_t$~\citep{Norris}. The probability ratio $\frac{p_t(y)}{p_t(x)}$ is commonly referred to as the discrete score function and can be learned either directly by estimating
$s_t^\theta(x)_y \approx \frac{p_t(y)}{p_t(x)}$~\citep{SEDD,CSM} or indirectly by learning the posterior distribution
$p_{0\mid t}^\theta \approx q_{0\mid t}$ and then recovering the discrete score via $\frac{p_t(y)}{p_t(x)} = \sum_{x_0} 
\frac{q_{t\mid 0}(y \mid x_0)}{q_{t\mid 0}(x \mid x_0)} q_{0\mid t}(x_0 \mid x)$~\citep{D3PM,TLDR}. Sampling from the CTMC reversal can be performed by discretizing time, the details of which we leave to Appendix \ref{section:background}. Since both the estimate of the discrete score function and the method of sampling are inexact, sampling accuracy depends on two distinct sources of error. Understanding the role of each source of error is the focus of our analysis.

\noindent {\bf Choice of corruption process.} Discrete diffusion models admit flexibility in the choice of $R_t$; however, for computational tractability over large state spaces, rate matrices with low-rank structure are typically used. Specifically, we focus on rate matrices defined for single coordinates $R_t = \beta_t(\mathbf{1}\vpi^T - I)$, where $\vpi \in \mathbb{R}^S$ denotes the stationary distribution of the corruption process and $\beta_t$ denotes the noise schedule. Two commonly used stationary distributions we study in this work are the uniform and masking distributions, $\vpi = \frac{1}{S}\mathbf{1}$ and $\vpi = \ve_M$.

\noindent {\bf Adjusting stochasticity at inference time.} Although fully deterministic sampling is generally infeasible in discrete settings, the reverse process can be made nearly deterministic~\citep{DFOT,MMF,DFM}. The Discrete Probability Flow (DPF) formulation removes \emph{redundant transitions} between state pairs:
\begin{align}
    \hat{R}_{t,\mathrm{DPF}}(y,x) 
    &= \bigl(\tfrac{p_t(x)}{p_t(y)}R_t(x,y) - R_t(y,x)\bigr)_+,
    \label{eq:DPF}
\end{align}
where $(\cdot)_+ = \max\{\cdot, 0\}$. 
Importantly, the learned discrete score function $\frac{p_t(y)}{p_t(x)}$ can be reused without any retraining. This construction eliminates symmetric transitions that exchange probability mass between states, resulting in faster but more deterministic dynamics. Although these redundant transitions do not affect the marginal evolution of $p_t$, we show that they play a critical role in correcting sampling errors. More generally, stochasticity can be reintroduced by adding any valid rate matrix $\hat{R}_{t,\mathrm{DB}}$ that satisfies the detailed-balance relation $p_t(x)\hat R_{t,\mathrm{DB}}(x,y) = p_t(y)\hat R_{t,\mathrm{DB}}(y,x)$:
\begin{align}
    \hat{R}_{t,\nu_t} &= \hat{R}_{t,\mathrm{DPF}} + \nu_t\hat{R}_{t,\mathrm{DB}},
    \label{eq:corrector}
\end{align}
where $\nu_t \geq 0$ controls the level of additional stochasticity~\citep{MMF}.

Under an exact discrete score, all such choices of $\nu_t$ produce the same marginal distributions. In practice, model approximation errors break this equivalence and hence $\nu_t$ must be tuned to balance sample quality and sampling efficiency.

\vspace{-.5em}
\section{Explaining Error Correction}
\vspace{-.5em}
\label{section:tradeoff}
A central question in discrete diffusion is: \emph{why does stochasticity improve sample quality, despite slowing down sampling?} Empirically, we observe a consistent tradeoff: less stochastic samplers converge faster but are sensitive to accumulated errors, while more stochastic samplers converge more slowly yet often achieve lower final error. This suggests that stochasticity plays a dual role, simultaneously affecting both convergence speed and error propagation. 

We show that this behavior is a consequence of how stochasticity interacts with different forms of sampling error across NFE regimes: (i) \textbf{Discretization errors (low-NFE)}: Increasing stochasticity increases the number of transitions per unit time, leading to larger discretization errors when using finite step sizes. (ii) \textbf{Model approximation errors (high-NFE)}: Stochastic transitions mix probability mass across states, contracting distributional errors and mitigating sampling errors. This section formalizes this tradeoff and characterizes the error-correcting role of stochasticity.

\subsection{Information-Theoretic Analysis}
\label{section:infotheory}
We first address how stochasticity affects sampling in the absence of model approximation error. In the low-NFE regime, increasing stochasticity leads to larger discretization error: Euler steps using rate $R$ are known to incur $O(\Delta t^2\Vert R\Vert^2)$ error locally. Intuitively, higher transition rates induce more frequent state changes, making finite-step approximations less accurate. This reasoning is consistent with prior work~\citep{TLDR,StochasticIntegral}, which provides global convergence guarantees in distribution with discretization errors that scale with the largest entries of $\hat{R}_t$. 

In the high-NFE regime, discretization errors are negligible, so improvements in sample quality due to stochasticity must arise from differences in how \emph{model approximation errors} accumulate. This suggests that stochasticity must play some role in enabling robustness to general errors.

To analyze error propagation, we consider $q_s$, a perturbed distribution relative to the true marginal $p_s$. A common intuition is that a noisy channel $q_{t\mid s}$ can only destroy distinguishability between inputs $p_s$ and $q_s$. This is captured through the data processing inequality, $D_{\mathrm{KL}}(q_{t\mid s}q_s\Vert q_{t\mid s}p_s) \leq D_{\mathrm{KL}}(q_s\Vert p_s)$, which implies that errors \emph{cannot become worse}. More importantly, under specific inputs and channels, the divergence can \emph{strictly contract}~\citep{SDPI,SDPIRaginsky}. This motivates studying the tightest possible contraction coefficient
\begin{align*}
    \eta_{\mathrm{KL}}(p_s,q_{t\mid s}) &= \sup_{q_s:D_{\mathrm{KL}}(q_s\Vert p_s) < \infty}\frac{D_{\mathrm{KL}}(q_{t\mid s}q_s \Vert q_{t\mid s}p_s)}{D_{\mathrm{KL}}(q_s \Vert p_s)},
\end{align*}
which quantifies the worst-case error propagation through the transition kernel, where $\eta_{\mathrm{KL}}(q_{t\mid s}) = \sup_{p_s}\eta_{\mathrm{KL}}(p_s,q_{t\mid s})$ takes the worst case over all input distributions. A natural question under this theoretical formulation is whether some corruption processes are more contractive. We show that this is indeed the case in Lemma \ref{thm:contraction}:

\begin{lemma}[Forward Process Contraction]
    \label{thm:contraction}
    Let $p_s$ be the ground-truth time marginal distribution and $q_s$ be a perturbed distribution. Then 
    \begin{align*}
        D_{\mathrm{KL}}(q_{t\mid s}q_s\Vert q_{t\mid s}p_s) &\leq \eta_{\mathrm{KL}}(p_s,q_{t\mid s})D_{\mathrm{KL}}(q_s\Vert p_s), 
    \end{align*} 
    where for masking diffusion we have $\eta_{\mathrm{KL}}(p_s,q_{t\mid s}^M) = \frac{\alpha_t}{\alpha_s}$, whereas for uniform diffusion we have $\eta_{\mathrm{KL}}(p_s,q_{t\mid s}^U) \leq \frac{S(\frac{\alpha_t}{\alpha_s})^2}{(S-2)\frac{\alpha_t}{\alpha_s}+2}$. In particular, for forward processes of the form $\vx_t = \alpha_t\vx_0 + (1-\alpha_t)\vpi$, masking achieves the worst-case contraction rate for all forward processes we consider.
\end{lemma}
The lemma shows that the error-contracting effects of stochasticity depend on the choice of transition kernel. Especially when the size of the state space $S$ is small, uniform diffusion exhibits stronger contraction than masking diffusion, where $\eta_{\mathrm{KL}}(p_s,q_{t\mid s}^U) < \eta_{\mathrm{KL}}(p_s,q_{t\mid s}^M)$. We further demonstrate that masking diffusion achieves the worst-case contraction rate for all inputs and the forward processes considered here (see Appendix~\ref{section:theory} for proof).

While the forward process is contractive toward a stationary distribution by design, it is less obvious whether similar guarantees hold for the reverse process used in sampling. Surprisingly, we show that reverse-time transitions also exhibit error contraction in Lemma \ref{thm:errorcorrection}, which states:
\begin{lemma}[Reverse Process Contraction]
\label{thm:errorcorrection}
Consider a diffusion process with noise schedule $\alpha_t$ and define $\sigma_{s\mid t}:=\frac{1-\alpha_s}{1-\alpha_t}$.
Let $q_{s\mid t}^{\mathrm{DPF}}$ denote the reverse transition kernel corresponding to the conditional DPF update in Equation~(\ref{eq:DDIM}). Then DPF sampling satisfies $\eta_{\mathrm{KL}}(p_t,q_{s\mid t}^{\mathrm{DPF}})
\le
\sigma_{s\mid t} + (1-\sigma_{s\mid t})\eta_{\mathrm{KL}}(p_t,p_{0\mid t}).$
For masking diffusion $\eta_{\mathrm{KL}}(q_{s\mid t}^{\mathrm{DPF},M}) = 1$, whereas for uniform diffusion $\eta_{\mathrm{KL}}(q_{s\mid t}^{\mathrm{DPF},U}) < 1$. Now consider a stochasticity schedule $\nu_t$ with corresponding reverse kernel $q_{s\mid t}^{\nu_t}$. This kernel can be expressed as a composition of $q^{b}_{s\mid t}$, the DPF step followed by the backward corrector step, and $q^{f}_{t\mid s}$, the forward step. Consequently, $\eta_{\mathrm{KL}}(p_t,q_{s\mid t}^{\nu_t})
\le\eta_{\mathrm{KL}}(p_t,q_{s\mid t}^{b})
\eta_{\mathrm{KL}}(q_{t\mid s}^{f})$. 
\end{lemma}
The result is twofold. First, it demonstrates that near the data distribution where the posterior is nearly deterministic, error contraction is significantly worse. Second, it shows that increasing stochasticity can reduce the contraction coefficient, correcting for perturbations from previous iterations. However, as we show next through an analysis in TV distance, this comes at the cost of amplifying model approximation errors in later iterations. 
\begin{theorem}[Contraction Under Model Error]
\label{thm:errorbound} Consider reverse-time sampling with a learned posterior estimator $p^\theta_{0\mid t}$ and discretization
$1=t_1 > t_2 > \cdots > t_{K+1}=0.$ Let $p_t^\theta$ denote the generated marginal distribution and $p_t$ the ground-truth marginal distribution. Define the posterior estimation error at step $k$ as $\epsilon_k
:=
D_{\mathrm{TV}}\left(
p^\theta_{0\mid t_k},
p_{0\mid t_k} \,\bigg|\, p_{t_k}^{\theta}\right).$ Then the sampling error satisfies the recursion
$
D_{\mathrm{TV}}(p^\theta_{t_{k+1}}, p_{t_{k+1}})
\le
A_k
D_{\mathrm{TV}}(p^\theta_{t_k}, p_{t_k})
+
B_k \epsilon_k,$
where \(A_k\) controls contraction of accumulated sampling error, \(B_k\) controls amplification of posterior estimation error. Consequently,
\[
D_{\mathrm{TV}}(p^\theta_0, p_0)
\le
\sum_{k=1}^K
\left(
\prod_{j=k+1}^K A_j
\right)
B_k \epsilon_k.
\vspace{-.5em}
\]  
For DPF sampling and sampling with stochasticity schedule $\nu_{t_k} = \frac{\alpha_{t_k}(1-\alpha_{t_{k+1}})}{\alpha_{t_{k+1}}-\alpha_{t_{k}}}$, these coefficients are
\begin{alignat*}{2}
A_k^{\mathrm{DPF}}
&=
\sigma_{t_{k+1}\mid t_k}
+
\left(1-\sigma_{t_{k+1}\mid t_k}\right)
\eta_{\mathrm{TV}}(p_{0\mid t_k}),
\qquad&
B_k^{\mathrm{DPF}}
&=
1-\sigma_{t_{k+1}\mid t_k},
\\
A_k^{\nu_t}
&=
\alpha_{t_{k+1}}
\eta_{\mathrm{TV}}(p_{0\mid t_k}),
\qquad&
B_k^{\nu_t}
&=
\alpha_{t_{k+1}}.
\end{alignat*}
In particular, we have $A_k^{\nu_t} \leq A_k^{\mathrm{DPF}}$ and $B_k^{\nu_t} \geq B_k^{\mathrm{DPF}}$.
\end{theorem}
\vspace{-.5em}
The bound reveals two competing terms: $A_k$, which contracts accumulated errors and decreases with increased stochasticity, and $B_k$, which amplifies model approximation error and grows with stochasticity. This formalizes the tradeoff observed in practice: stochasticity improves sample quality up to an optimal level, beyond which model error dominates. Finally, although our analysis focuses on the one-dimensional setting,
Appendix~\ref{section:theory} extends the error bounds to general
dimensions, recovering them in terms of KL contraction coefficients and
showing that contraction is most effective near the noise distribution
when input dimensions are nearly independent.
\vspace{-.5em}
\subsection{1D Experimental Validation}
To validate the theoretical predictions, we analyze a controlled 1D discrete diffusion model with 15 states under a uniform corruption process. Using an analytically derived posterior $p_{0\mid t}$, this setting isolates sampling dynamics from model approximation error, allowing us to directly study the effect of stochasticity~\citep{RKT}. We compare $\tau$-leaping using the default stochastic rate matrix and the DPF rate matrix, which we refer to as $\tau$-leaping and DPF sampling, respectively. Figures~\hyperref[fig:1D]{2a} and~\hyperref[fig:1D]{2b} show that, in the absence of model error, DPF achieves lower error faster than $\tau$-leaping, consistent with reduced stochasticity leading to fewer transitions and lower discretization error.

To study error propagation, we introduce model approximation errors by perturbing the score function with multiplicative noise (see Appendix~\ref{section:1D} for details). In this setting, we observe a reversal in behavior: DPF sampling becomes more sensitive to errors and ultimately performs worse than $\tau$-leaping. This supports our theoretical claim that stochastic transitions provide an error-correcting effect that mitigates accumulated errors over time.

\begin{figure*}[t]
    \centering
    \includegraphics[width=\linewidth]{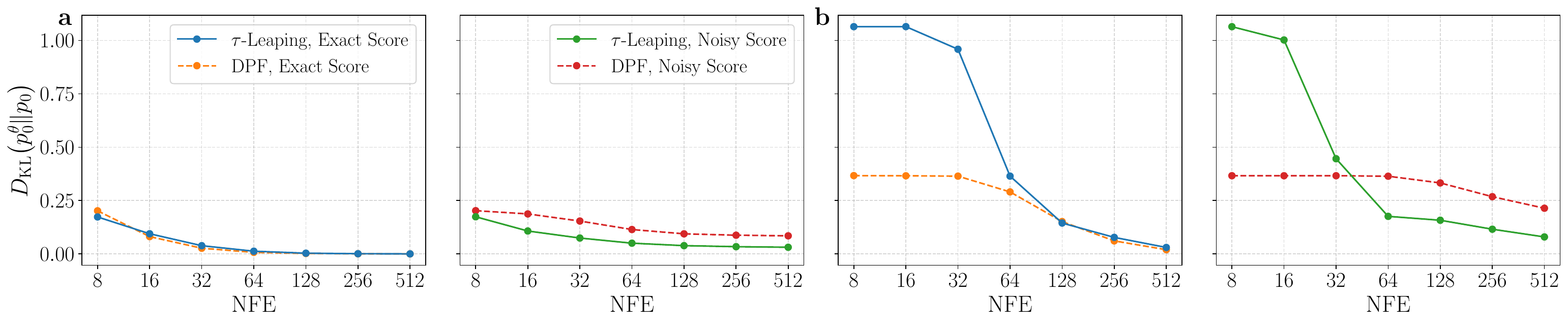}
    \vspace{-1em}
    \caption{{\bf Controlled 1D experiments illustrating the speed-quality tradeoff and error-correcting effect of stochasticity.} We report empirical KL divergence between the ground-truth distribution $p_{0}$ and generated distribution $p_0^{\theta}$ under a uniform corruption process. Results are shown for different noise schedules $\beta_t$: ({\bf a}) linear and ({\bf b}) geometric, and each schedule is evaluated over exact scores (left) and perturbed scores (right). With exact scores, near-deterministic sampling (DPF) converges faster due to reduced stochasticity. Under perturbed scores, stochastic sampling ($\tau$-leaping) achieves lower final error, demonstrating its ability to correct accumulated model errors. The gap is more pronounced under the geometric 
    schedule, which drives $\sigma_{s\mid t}$ closer to $1$ near the noise distribution, weakening DPF error correction while making the benefits of stochastic mixing more apparent.}
    \label{fig:1D}
    \vspace{-1em}
\end{figure*}

We further examine the role of the noise schedule in light of our analysis in Theorem~\ref{thm:errorbound}. Under a geometric schedule $\alpha_t$, we expect the DPF contraction factor $\sigma_{s\mid t} \approx 1$ near the noise distribution, consequently leading to greater error accumulation under DPF sampling. As a result, the additional error correction from stochasticity in $\tau$-leaping becomes more pronounced, leading to the observed divergence between $\tau$-leaping and DPF.

Overall, these results illustrate the two competing effects identified in Section~\ref{section:infotheory}: reduced stochasticity improves convergence speed but increases sensitivity to model error, while higher stochasticity increases discretization errors but performs robustly at high NFE due to error-correcting effects.

\vspace{-.5em}
\section{Discrete Churn and Restart Sampling (DCRS)}
\vspace{-.5em}
Building on the analysis in Section~\ref{section:tradeoff}, we now design a sampling algorithm that explicitly exploits the error-correcting role of stochasticity. Our results show that stochastic transitions contract errors but slow convergence, while near-deterministic transitions accelerate sampling but are prone to error accumulation. This indicates that an effective sampler should combine both behaviors: primarily deterministic updates for efficiency, together with strategically injected stochasticity for error correction. We propose \emph{Discrete Churn and Restart Sampling} (DCRS), an inference algorithm that interleaves near-deterministic reverse updates with controlled injections of forward-process noise. DCRS consists of two complementary mechanisms: \emph{restarting}, which introduces global stochasticity across the trajectory, and \emph{churning}, which provides local error correction within a time window.

\textbf{Restarting.} DCRS initiates the generative process by simulating a reverse trajectory from the stationary distribution at $t=1$ to a predefined time $t_{\min}$ to compute an initial state $x_{t_{\min}}^{0} = \text{DPF}_{\theta}(x_1, 1 \to t_{\min})$, where $\text{DPF}_{\theta}(x, t \!\to\! s)$ denotes running the DPF sampler from state $x$ over the interval $[s,t]$. The algorithm then enters a \emph{restart} phase, alternating between forward-time corruption and reverse-time updates for $i \in \{0,\ldots,K-1\}$ iterations: 
\begin{align*}
    x_{t_{\max}}^{i+1} \sim q_{t_{\max} \mid t_{\min}}(\cdot \mid x_{t_{\min}}^{i}), \text{  and }
    x_{t_{\min}}^{i+1} = \text{Churn}_{\theta}(x_{t_{\max}}^{i+1}, t_{\max} \to t_{\min}).
\end{align*} 
Here, $\text{Churn}_{\theta}(x, t \!\to\! s)$ refers to the corresponding churning update at state $x$ over the interval from $t$ to $s$ (see below). Because the forward transition $q_{t_{\max}\mid t_{\min}}$ admits a closed-form expression, these restart steps introduce stochasticity without additional neural network evaluations or model approximation error, providing an efficient mechanism for error correction across the trajectory.

\textbf{Churning.} While restarting introduces stochasticity to correct errors before the restart, it does not address error accumulation within the restart window. To address this, we introduce a finer-grained mechanism termed \emph{churning}, which injects local stochasticity within each restart window. Given a time discretization $\{t_i\}_{i=1}^{N}$, the forward process advances to an intermediate time $(1+\gamma)t_i$ for some $\gamma \in (0,1)$:
\vspace{-.5em}
\begin{equation*}
    x_{(1 + \gamma)t_{i}} \sim q_{{(1 + \gamma)t_{i}}\mid t_{i}}(\cdot \mid x_{t_{i}}).
\end{equation*}
This injected stochasticity is followed by a single DPF sampling step from $(1+\gamma)t_i$ to $t_{i+1}$. In contrast to full restart iterations, which may involve multiple forward and reverse steps, the churning operation uses only one reverse update with a larger step size to compensate for the temporary forward movement in time. Churning allows for additional correction within the restart window, especially when the error accumulation under DPF sampling is significant. Generation finishes by running $x_0 = \text{DPF}_{\theta}(x_{t_{\text{min}}}^K,t_{\text{min}}\to 0)$. Together, restarting and churning allow DCRS to balance efficiency and error correction. Deterministic reverse steps provide fast convergence, while forward-process perturbations periodically reintroduce stochasticity to contract accumulated errors. Further details about the sampling algorithm can be found in Algorithm~\ref{alg:discrete_restart_sampling} of Appendix~\ref{section:alg}.

\noindent {\bf Impact of hyperparameters.} DCRS introduces three main hyperparameters: the restart interval $[t_{\min}, t_{\max}]$, the number of restart iterations $K$, and the churning parameter $\gamma$. These parameters control where and how stochasticity is injected, and therefore directly govern the tradeoff between error correction and discretization error. In particular, we find that hyperparameter selection strategies from continuous diffusion samplers generally do not transfer well to the discrete setting, due in part to parallel decoding errors across dimensions that are unique to discrete diffusion transitions.

Our theoretical analysis suggests that restarting should be most effective in regions where errors accumulate and DPF contraction is otherwise weak. Empirically, we find that optimal restart windows align with regions where discretization steps are coarsest, suggesting that the added stochasticity is effective at mitigating parallel decoding errors across dimensions.

By the same reasoning, increasing the number of restart iterations exhibits diminishing returns. While additional restarts initially improve sample quality, excessive restarts introduce additional discretization error and can ultimately degrade performance. In practice, we find that the largest marginal improvement is typically achieved after the first restart iteration.

Finally, the churning parameter $\gamma$ controls the magnitude of forward-reverse perturbations within the restart window. We observe that small values of $\gamma \approx 10^{-3}$ consistently yield the best performance, suggesting that beyond a certain scale, the discretization error induced by larger DPF steps outweighs the error-correcting benefits of increased stochasticity.

\noindent {\bf Applying higher-order solvers.} While higher-order solvers reduce discretization error, they do not address model approximation error and can exacerbate error accumulation under near-deterministic dynamics. To improve sample quality in large-scale datasets, we employ a higher-order solver for the DPF equation (Equation (\ref{eq:DPF})) only within the restart window. Specifically, we use the trapezoidal solver outlined in \citet{RKT}, which uses a weighted combination of two reverse rate matrix evaluations to reduce discretization errors. In practice, using higher-order solvers outside the restart window often leads to catastrophic compounding of discretization errors (see Appendix~\ref{section:failure}). To address this, we leverage the error-correcting effects of restarting and noise churning within the restart window. For additional details, we refer to Appendix~\ref{section:alg}. 
\vspace{-.5em}
\section{Experiments}
\vspace{-.5em}
To validate the theoretical predictions in Section~\ref{section:tradeoff}, we evaluate how stochasticity affects sampling behavior across controlled and large-scale settings. We focus on two questions: (1) whether stochasticity improves sample quality through error correction, and (2) whether DCRS can leverage this effect to improve the speed-quality tradeoff in practice.

\noindent {\bf Image datasets.} We use pretrained model checkpoints from \citet{TLDR} and \citet{DFOT} on CIFAR10 and CelebA, respectively, which employ a discretized Gaussian corruption process. We evaluate the quality of generated samples in terms of Fr\'echet Inception Distance (FID) over 50K samples (Figure~\hyperref[fig:overview]{1e}). For both datasets, we use a restart window of $[t_{\text{min}},t_{\text{max}}] = [0.7, 0.8]$ and test a single restart iteration over varying NFE used outside of the restart window. To demonstrate how DCRS scales with NFE, we proportionally increase the number of discretization steps in the restart window relative to the total number of steps outside the window. We also use $\tau$-leaping outside of the restart window, where a higher-order solver is used inside the window. We observe that replacing the DPF rate matrix with the modified matrix $\hat{R}_{t,\nu_t}$ with $\nu_t = 0.005$ tends to slightly boost performance (Equation (\ref{eq:corrector})). In practice, we observe that higher-order solvers are ineffective without the DPF rate matrix, and hence we do not provide such samplers as baselines in our results (Appendix~\ref{section:ablation}).

Overall, DCRS substantially improves the speed-quality tradeoff in the low-NFE regime, achieving comparable sample quality with significantly fewer model evaluations (20-30 steps). On CelebA, DCRS matches the best performance of $\tau$-leaping with fewer evaluations, while on CIFAR10, we observe a small gap at high NFE. These results are consistent with our theoretical analysis: reducing stochasticity improves efficiency but increases error accumulation, whereas additional stochasticity can improve robustness to model approximation errors that persist in the high-NFE regime.
\begin{table}[t]
\centering
\footnotesize
\begin{tabular}{lccccccccc}
\toprule
Method 
& \multicolumn{3}{c}{MAUVE ($\uparrow$)} 
& \multicolumn{3}{c}{Gen PPL. ($\downarrow$)} 
& \multicolumn{3}{c}{Entropy ($\uparrow$)} \\
\midrule

% -------- high-NFE --------
\shadecolumn 
& \it{T=256} & \it{T=512} & \it{T=1024}   
& \shadecolumnlight \it{T=256} & \shadecolumnlight \it{T=512} & \shadecolumnlight \it{T=1024}  
& \it{T=256} & \it{T=512} & \it{T=1024} \\

\shadecolumn ReMDM-Cap 
& 0.00635 & 0.00618 & 0.00628 
& \shadecolumnlight 175.44 & \shadecolumnlight 226.54 & \shadecolumnlight 356.62
& 5.749 & 5.797 & 5.875 \\

\shadecolumn ReMDM-Loop 
& 0.00697 & 0.00785 & 0.00714 & 
\shadecolumnlight 135.31 & \shadecolumnlight 141.15 & \shadecolumnlight 175.02 
& 5.710 & 5.719 & 5.758 \\

\shadecolumn MDLM 
& 0.00762 & 0.00770 & 0.00822 
& \shadecolumnlight 112.20 & \shadecolumnlight 106.93 & \shadecolumnlight 105.07 
& 5.653 & 5.640 & 5.635 \\

\midrule

% -------- low-NFE --------
\shadecolumn 
& \it{T=32} & \it{T=64} & \it{T=128} 
& \shadecolumnlight \it{T=32} & \shadecolumnlight \it{T=64} & \shadecolumnlight \it{T=128} 
& \it{T=32} & \it{T=64} & \it{T=128} \\

\shadecolumn ReMDM-Cap 
& 0.00489 & 0.00567 & 0.00603
& \shadecolumnlight 224.55 & \shadecolumnlight 172.19 & \shadecolumnlight 160.28 
& 5.762 & 5.734 & 5.728 \\

\shadecolumn ReMDM-Loop 
& 0.00453 & 0.00501 & 0.00615
& \shadecolumnlight 349.25 & \shadecolumnlight 198.74 & \shadecolumnlight 148.71 
& 5.814 & 5.758 & 5.722 \\

\shadecolumn MDLM 
& 0.00499 & 0.00556 & 0.00745
& \shadecolumnlight 196.63 & \shadecolumnlight 143.30 & \shadecolumnlight 121.67
& 5.738 & 5.697 & 5.669 \\ 
\bottomrule
\end{tabular}
\vspace{1em}
\caption{Generative perplexity, entropy, and MAUVE evaluated for a masked diffusion language model (MDLM) trained on OpenWebText over 5000 sequences of length 1024.}
\vspace{-3em}
\label{tab:MDLM}
\end{table}

\noindent {\bf Language datasets.} We evaluate pre-trained models from UDLM~\citep{UDLM} and MDLM~\citep{MDLM} on LM1B and OpenWebText, which use uniform and masking corruption processes, respectively. For MDLM, we examine the stochasticity tradeoff from adding stochasticity via ReMDM, which adjusts stochasticity through closed-form transition distributions (see Appendix~\ref{section:remdm}). We report generative perplexity for the uniform diffusion model over 1024 sequences of length 128, and for the masked diffusion model over 5000 sequences of length 1024. 

In contrast to image models, we find that reducing stochasticity has a limited impact on final sample quality. In particular, our results for UDLM (see Appendix~\ref{section:lang}, Table~\ref{tab:UDLM}) show that deterministic and stochastic samplers (e.g., DDIM vs. D3PM) achieve comparable performance. This can be explained by the fact that these samplers avoid constant-rate approximations, further suggesting that parallel decoding errors serve as the major performance bottleneck. As a result, increasing stochasticity alone does not improve convergence in this setting. Broadly, these results suggest that improving sample quality in text diffusion models may require mechanisms that explicitly introduce transitions across dimensions to better correct for parallel decoding errors, rather than solely increasing stochasticity within each dimension.

We further observe that, while less stochastic samplers achieve lower error at low NFE, increased stochasticity does not consistently improve sample quality and can even degrade performance. In Table~\ref{tab:MDLM}, we show that the error-correcting effects reported in ReMDM disappear without nucleus sampling. We hypothesize that nucleus sampling sharpens the posterior $p_{0\mid t}$, biasing samples toward high-probability modes and altering the contraction behavior of the diffusion process. Consistent with our theory, this sharpening can reduce the effective error-correcting ability of stochastic transitions by pushing the system toward a worst-case contraction regime.

\begin{figure*}[htbp]
    \centering
    \includegraphics[width=\linewidth]{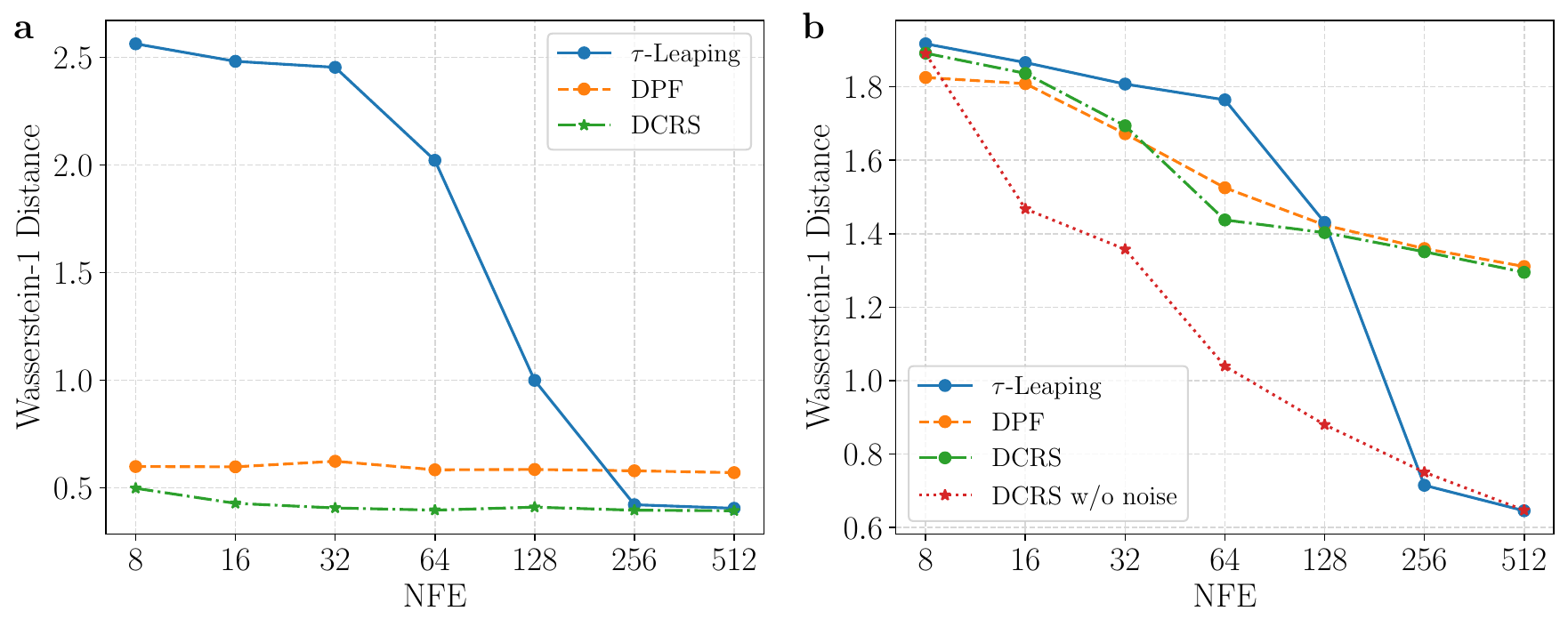}
    \caption{
    {\bf Sample generation experiments on an 8D projection of a mixture of Gaussians dataset for models trained on \textbf{(a)} uniform and \textbf{(b)} masking corruption processes.} Sample quality is evaluated using the Wasserstein-1 distance between 10K generated samples and the ground-truth distribution. DCRS achieves the best speed-quality tradeoff through error correction.
    }
    \label{fig:projected}
\end{figure*}
\noindent {\bf Mixture of Gaussians.} We study a controlled mixture-of-Gaussians setting to directly visualize the effect of stochasticity on mode coverage and error correction. We train MLP-based diffusion models on a 2D dataset with 8 modes and, following prior work~\citep{DFOT,SDDM}, quantize the data into bit strings. We add high-dimensional structure by projecting the data into 8 dimensions using a random Gaussian projection, yielding 128-bit sequences (see Appendix~\ref{section:details}). To accentuate model errors, we train models with a linear noise schedule and perform inference with a geometric schedule, which increases step sizes near the data distribution and amplifies error accumulation. 

In Figure~\hyperref[fig:projected]{3a}, we observe a clear tradeoff between convergence speed and sample quality for the uniform diffusion model. The $\tau$-leaping sampler converges slowly but produces well-distributed samples across modes, whereas the near-deterministic DPF sampler immediately generates samples near the ground-truth distribution (within $\sim$8 NFE) but fails to improve with NFE, reflecting sensitivity to accumulated errors. In contrast, DCRS restores sample quality by reintroducing stochasticity. 

We evaluate sample quality using the Wasserstein-1 distance between 10K generated samples and the ground-truth distribution. Across both uniform and masking corruption processes, DCRS achieves improved speed-quality tradeoffs by balancing fast convergence with error correction. For DCRS, we use restart windows of $[t_{\text{min}}, t_{\text{max}}] = [0.01, 0.05]$ (uniform) and $[0.1, 0.2]$ (masking), with the number of discretization steps inside the restart window matched to the main trajectory. In this controlled setting, we find that restarting alone provides most of the benefit, while additional noise churning or higher-order solvers can introduce unnecessary discretization error.

\noindent {\bf Impact of temperature sampling.} For the masked diffusion experiments in Figure~\hyperref[fig:projected]{3b}, we sample from the sharpened score $(\frac{p_t(y)}{p_t(x)})^{1/\tau}$ for $\tau = 0.8$. The overall effect is equivalent to scaling marginal logits $\tau^{-1}\log p_{t}(y)$, analogous to temperature sampling. We observe that stochasticity alone does not yield error-correcting benefits in this setting; instead, these effects emerge only in conjunction with temperature scaling. This can be explained by the structure of masking diffusion: without additional stochasticity, sequences remain highly masked until late in the trajectory ($t \approx 0.1$), limiting opportunities for mixing and error contraction. Temperature scaling sharpens the posterior, further concentrating probability mass and reducing the effectiveness of stochastic transitions in correcting errors. To counteract this, we introduce targeted stochasticity using a schedule $\nu_t = 20 \cdot \mathbf{1}_{t < 0.1}$, which improves sample quality at high NFE by encouraging mixing earlier in the trajectory. However, we find that restarting alone, without additional noise, often yields stronger improvements than combining restarting with stochastic perturbations. This suggests that resetting the trajectory to earlier time steps allows the reverse CTMC to revisit higher-entropy regions, where contraction is stronger and mixing between modes is more effective.

\section{Related Work}
\paragraph{Accelerating Discrete Diffusion Inference.} Prior work has achieved strong performance in the low-NFE regime of discrete diffusion models through \emph{distillation methods} \citep{Duality,Di4C,SDTT}, which require retraining models using pre-trained samplers as targets. In contrast, our work focuses on improving the efficiency of \textit{training-free} sampling methods. A complementary line of work develops higher-order solvers that approximate transition dynamics using Taylor expansions, including Runge-Kutta and trapezoidal methods \citep{RKT}. These approaches trade additional neural network evaluations for reduced discretization error. Corrector-based methods \citep{TLDR,Informed} further improve sampling by introducing Markov Chain Monte Carlo updates that preserve the time marginals. Unlike these approaches, which primarily address discretization error, our method targets the accumulation of \emph{model approximation errors} through controlled stochasticity. Notably, higher-order solvers and corrector iterations are complementary to DCRS and can be combined to further improve performance.
\vspace{-.5em}
\paragraph{Understanding Error Correction in Diffusion Models.} In continuous diffusion models, stochasticity has been observed to improve sample quality by mitigating error accumulation, for example, through noise injection during sampling \citep{EDM,Restart}. Theoretical work \citep{Restart} shows that stochasticity can induce contraction in the Wasserstein-1 distance. In discrete settings, prior work \citep{UDLM,ReMDM,GIDD} has noted that uniform diffusion processes can iteratively correct errors, while masking-based processes involve irreversible transitions. Similar observations have been made in discrete flow-matching, where redundant transitions can improve performance. However, these effects are largely understood heuristically. In contrast, we provide a principled characterization of error correction in discrete diffusion through KL contraction analysis and show how stochasticity can be explicitly controlled to mitigate the accumulation of \emph{model approximation errors} during inference. We demonstrate that the effectiveness of stochasticity depends on the structure of the corruption process and validate these findings empirically across controlled and large-scale settings.

\section{Conclusion}
We investigated the role of stochasticity in discrete diffusion models and showed that it induces an error-correcting mechanism that mitigates the accumulation of model approximation errors during sampling. Building on this insight, we proposed Discrete Churn and Restart Sampling (DCRS), a training-free algorithm that balances fast deterministic updates with strategically injected stochasticity to improve the speed-quality tradeoff, particularly in the low-NFE regime. Our theoretical and empirical results highlight that the effectiveness of stochasticity depends on both its magnitude and the structure of the corruption process, suggesting that controlling stochasticity is a fundamental design principle for discrete diffusion inference. A limitation of our approach is the need to tune the level and schedule of stochasticity. An interesting direction for future work is to develop adaptive methods for optimizing this tradeoff, as well as to further characterize optimal sampling rates through information-theoretic analysis.

\bibliographystyle{unsrtnat}
\bibliography{./neurips_conference}

\appendix
\section{CTMC Background}
\label{section:background}
\subsection{CTMC Simulation Details}
Given $R_t$, one can approximate the forward process of the CTMC by discretizing time with step size $\Delta t$, yielding the infinitesimal transition distribution:
\begin{equation}
\label{eq:inf_transition}
    q_{t\mid t-\Delta t}(y \mid x) = \delta_{x,y} + R_t(x,y)\Delta t + O(\Delta t^2),
\end{equation}
where $\delta_{x,y}$ denotes the indicator function over $x = y$. For computational convenience, the rate matrix is often restricted to the form $R_t = \beta_t R_b$, where $R_b \in \mathbb{R}^{S^D\times S^D}$ is a fixed base rate matrix and $\beta_t$ defines the noise schedule. Under this assumption, the forward transition distribution admits the following closed-form expression:
\begin{equation}
\label{eq:forward}
    q_{t\mid s}(y \mid x) = \Bigg[\exp\!\Big(\int_s^t R_\tau \, d\tau\Big)\Bigg](x,y).
\end{equation}

Because the state space $\gS^D$ grows exponentially with dimension, the full rate matrix is often intractable in practice. Consequently, the reverse rate matrix is typically structured to factor across dimensions:
\begin{equation}
    \hat{R}_t(x,y) 
= \sum_{d=1}^{D} \mathbf{1}_{\{x^{\setminus d}=y^{\setminus d}\}} \frac{p_t(y^d \mid x^{\setminus d})}{p_t(x^d \mid x^{\setminus d})}R_t^{d}(y^d,x^d),
\label{eq:factorization}
\end{equation}
where $x^{\setminus d}$ denotes all coordinates except $d$, and $R_t^{d} \in \mathbb{R}^{S \times S}$ is a single-site rate matrix acting on dimension $d$. Under this factorization, the reverse process is a composition of conditionally independent processes governed by $R_t^{d}$. Sampling can therefore be performed independently for each dimension, incurring at most second-order discretization error:
\begin{equation}
    q_{t\mid t-\Delta t}(y^{1:D}\mid x^{1:D}) = \prod_{d=1}^{D}q_{t\mid t-\Delta t}(y^d\mid x^{1:D}) + O(\Delta t^2).
\label{eq:parallel}
\end{equation}

Herein, we divide the main algorithms for sampling into discretization samplers, such as $\tau$-leaping and Euler steps that assume a constant rate matrix over each sampling step, and analytic samplers, such as Tweedie $\tau$-leaping and the D3PM posterior.
\subsubsection{Discretization Samplers}
In principle, samples can be generated via discretizing the infinitesimal transition distribution defined in Equation (\ref{eq:inf_transition}), using a first-order linearization that assumes a constant rate matrix over some interval of length $\Delta t$.

When the state space lacks ordering, such as in language modeling, the naive approach of simulating the CTMC through sampling from the categorical distribution defined by $\delta_{x_t,y} + \Delta t R_t(x_t,y)$ corresponds to the current state $x_t$. If step sizes $h$ are not sufficiently small, the computed distribution may not be a valid distribution, and care must be taken to adaptively select interval sizes to avoid this~\citep{DFM}.

When the state space is well-ordered, such as the pixel values of an image, $\tau$-leaping can improve upon the efficiency of Euler steps by simulating transitions via independent Poisson arrivals across dimensions. Specifically, samples are updated according to:
\begin{equation*}
    x_{t-\Delta t} = x_t + \sum_{y \in B_1(x_t)} P_{t,y} (y - x_t),
\end{equation*}
where $B_1(x_t)$ denotes the Hamming-1 neighborhood of $x_t$ and
$P_{t,y} \sim \text{Poisson}(\Delta t\, \hat{R}_t(x_t,y))$. Crucially, $\tau$-leaping sums multiple jumps together, allowing for more efficient traversal of the state space.

\subsubsection{Analytic Samplers}
D3PM uses Bayes' rule to derive a marginally consistent time-reversal of the discrete time forward process, given by
\begin{align}
    q(x_s\mid x_t,x_0) &= \frac{q(x_t\mid x_s)q(x_s\mid x_0)}{q(x_t\mid x_0)}.
    \label{eq:bayes}
\end{align}

Using an estimate of the posterior $p_{0\mid t}$, $x_0$ can be marginalized out to sample from the closed-form transition distribution. When only evaluation of the discrete score function is possible,~\citep{SEDD} proposed Tweedie $\tau$-leaping, which uses the score function and the known forward transition distribution to reconstruct the posterior in Equation (\ref{eq:bayes}).

Although the closed-form transition distribution does not suffer from discretization errors incurred from assuming a constant rate matrix, the samplers still incur $O(\Delta t^2)$ errors from sampling multiple dimensions independently in parallel (Equation (\ref{eq:factorization},\ref{eq:parallel})).

\section{Theoretical Results}
\label{section:theory}
\subsection{Discrete Probability Flow}
For clarity, we present a proof that the DPF rate matrix indeed preserves the time marginals in Theorem~\ref{thm:dpf} below:
\begin{theorem}[DPF Preserves Time Marginals]
    \label{thm:dpf}
    The DPF rate matrix given by Equation (\ref{eq:DPF}) preserves time marginals $p_t$ generated by the forward rate matrix $R_t$.   
\end{theorem}
\begin{proof}
The time marginal distributions $p_t$ are generated from $R_t$ if and only if they satisfy the Kolmogorov Forward Equation
\begin{align}
    \left[\frac{dp_t}{dt}\right]_x &= [R_t^Tp_t]_x = \sum_{y\neq x}\big(p_t(y)R_t(y,x) - p_t(x)R_t(x,y)\big).
    \label{eq:KFE}
\end{align}  

Under the time-reversal relation, we know that the forward rate matrix corresponding to the DPF matrix satisfies:
\begin{align}
    R_{t,DPF}(x,y) &= \frac{p_t(y)}{p_t(x)}\hat{R}_{\mathrm{DPF}}(y,x) = \frac{(p_t(x)R_t(x,y)-p_t(y)R_t(y,x))_+}{p_t(x)}.
\end{align}

We then plug the DPF rate matrix into Equation \ref{eq:KFE}:
\begin{align*}
    \left[\frac{dp_t}{dt}\right]_x 
    &= \sum_{y\neq x}\bigg(p_t(y)\frac{(p_t(y)R_t(y,x)-p_t(x)R_t(x,y))_+}{p_t(y)}\bigg) \\
    &- \sum_{y\neq x}\bigg(p_t(x)\frac{(p_t(x)R_t(x,y)-p_t(y)R_t(y,x))_+}{p_t(x)}\bigg) \\
    &= \sum_{y\neq x}\big(p_t(y)R_t(y,x) - p_t(x)R_t(x,y)\big).
\end{align*}
\end{proof}
Although written in a different form, the DPF rate matrix coincides with the minimal stochasticity probability flows defined in~\citep{MMF,DFM,KODFM} for uniform and masking formulations. These works focus on more general settings where the time marginals may not necessarily be generated by a Markov chain.

When $R_t$ is symmetric, the rate matrices become
\begin{align}
    R_{\mathrm{DPF}}(x,y) &= \left(1-\frac{p_t(y)}{p_t(x)}\right)_+R_t(x,y), \\
    \hat{R}_{\mathrm{DPF}}(x,y) &= \left(\frac{p_t(y)}{p_t(x)} - 1\right)_+R_t(x,y).
\end{align}
which can be seen as a thresholding of the discrete score function such that transitions are suppressed when $p_t(y) \leq p_t(x)$. The result is a flow toward states with higher probability mass.

Defining the probability mass per unit time flowing from state $x$ to $y$ as $v_t(x,y) = p_t(x)R_t(x,y)$, we find that $\min\{v_t(x,y),v_t(y,x)\}$ equally flows into and out of state $x$ and $y$. The DPF rate matrix can be understood as removing redundant mutual positive flow, which only serves to swap probability mass equally between state pairs. The effective new flow is set to be the remainder $v_t(x,y) - \min\{v_t(x,y),v_t(y,x)\} = (v_t(x,y) - v_t(y,x))_+$.

In particular, masking processes do not have mutual flow between state pairs, whereas uniform processes do. In the event that mutual flow is zero, additional mutual flow can be added via the stochasticity schedule $\nu_t$ in Equation~\ref{eq:corrector}.

\subsection{DDIM / ReMDM}
\label{section:remdm}
In this section, we adopt the notation of $\vx_t\in\R^S$ to denote a probability vector representing the time marginal distribution $p_t$.

Under DDIM, the sampling update is
\begin{align}
    \label{eq:noisyDDIM}
    \vx_s &= \sigma_t\vx_t + (\alpha_s-\sigma_t\alpha_t)\vx_0 + ((1-\alpha_s)-(1-\alpha_t)\sigma_t)\vpi,
\end{align}
where $\sigma_t = \frac{1-\alpha_s}{1-\alpha_t}$ eliminates the $\vpi$ term and minimizes stochasticity. DDIM preserves the time marginals generated by the relation $\vx_t = \alpha_t\vx_0 + (1-\alpha_t)\vpi$, which can be verified by plugging in for $\vx_t$ and checking that $\vx_s$ is marginally consistent.

In Theorem~\ref{thm:dpfddim}, we demonstrate that under these corruption processes, DDIM corresponds to the posterior-averaged transition kernel of the conditional DPF:
\begin{theorem}[Conditional DPF and DDIM Equivalence]
    \label{thm:dpfddim}
    For corruption process 
    $\vx_t = \alpha_t\vx_0 + (1-\alpha_t)\vpi$,
    let $\hat{R}_{t,\mathrm{DPF}}(\cdot,\cdot\mid x_0)$ denote the conditional DPF rate matrix obtained by applying Equation~(\ref{eq:DPF}) to the conditional distribution $p_{t\mid0}(\cdot\mid x_0)$.
    The reverse transition distribution $q_{s\mid t}$ can be expressed in closed-form:
    \begin{equation}
    \label{eq:DDIM}
        \vx_s = \left(\frac{1-\alpha_s}{1-\alpha_t}\right)\vx_t
        + \left(1-\frac{1-\alpha_s}{1-\alpha_t}\right)\vx_0,
    \end{equation}
    which exactly matches the DDIM update~\citep{DDIM}
    outlined for the uniform diffusion case.
\end{theorem}
\begin{proof}
    Following the result of Theorem 3 of~\citep{DFM}, we know that the conditional rate matrix is given by
    \begin{align*}
        \hat{R}_{t,\mathrm{DPF}}(\cdot,\cdot\mid x_0)
        &= \frac{-\dot{\alpha}_t}{1-\alpha_t}\,(\mathbf{1}\ve_{x_0}^T-I),
    \end{align*}
    where $-\dot\alpha_t \ge 0$ since $\alpha_t$ is decreasing. 
    
    Integrating over
    the reverse interval from $t$ to $s<t$,
    \begin{align*}
        Q_{s\mid t,x_0}
        &= \exp\left(\left(\int_s^t\frac{-\dot{\alpha}_\tau}{1-\alpha_\tau}\,d\tau\right)
           (\mathbf{1}\ve_{x_0}^T-I)\right) \\
        &= \left(\frac{1-\alpha_s}{1-\alpha_t}\right)I
          + \left(1- \frac{1-\alpha_s}{1-\alpha_t}\right)\mathbf{1}\ve_{x_0}^T,
    \end{align*}
    using $(\mathbf{1}\ve_{x_0}^T - I)^2 = -(\mathbf{1}\ve_{x_0}^T - I)$.

    Averaging over $p_{0\mid t}$, we obtain the marginal
    transition matrix
    \begin{align*}
        Q_{s\mid t}
        &= \left(\frac{1-\alpha_s}{1-\alpha_t}\right)I
        + \left(1-\frac{1-\alpha_s}{1-\alpha_t}\right)P_{0\mid t},
    \end{align*}
    where $P_{0\mid t}$ denotes the posterior transition matrix.
    
    The sampling update is then given by
    \begin{align*}
        \vx_s &= Q_{s\mid t}^T\vx_t
        = \left(\frac{1-\alpha_s}{1-\alpha_t}\right)\vx_t
        + \left(1-\frac{1-\alpha_s}{1-\alpha_t}\right)\vx_0,
    \end{align*}
    where $\vx_0 = P_{0\mid t}^T\vx_t$.
\end{proof}
The ReMDM sampler matches the DDIM update in Equation (\ref{eq:noisyDDIM}), under the reparameterization $\sigma_t = \frac{1 - \alpha_s - \sigma_t^{\mathrm{ReMDM}}}{1-\alpha_t}$. Under this choice, the update becomes
\begin{align*}
    \vx_s &= \left(\frac{1 - \alpha_s - \sigma_t^{\mathrm{ReMDM}}}{1-\alpha_t}\right)\vx_t + \left(\frac{\alpha_s - \alpha_t + \sigma_t^{\mathrm{ReMDM}}\alpha_t}{1-\alpha_t}\right)\vx_0 + \sigma_t^{\mathrm{ReMDM}}\vpi,
\end{align*}
where $\sigma_t^{\mathrm{ReMDM}}$ controls the probability of sampling from the stationary distribution, or the masked state.

Importantly, we also demonstrate that under added stochasticity $\sigma_t$, DDIM / ReMDM decomposes into a DDIM / DPF step, followed by a corrector step that preserves time marginals. According to Theorem 4.1 of~\citep{ReMDM}, every update is a composition of Equation (\ref{eq:DDIM}) and the following corrector step that preserves the marginal at time $s$:
\begin{align}
    \label{eq:pc}
    \vx_{s_c} &= \left(1-\frac{\sigma_t^{\mathrm{ReMDM}}}{1-\alpha_s}\right)\vx_s + \frac{\sigma_t^{\mathrm{ReMDM}}\alpha_s}{1-\alpha_s}\vx_0 + \sigma_t^{\mathrm{ReMDM}}\vpi.
\end{align}
In general, we can also decompose the corrector step according to
\begin{align}
    \vx_{s_b} &= \left(1-\frac{\sigma_t^{\mathrm{ReMDM}}\alpha_s}{(1-\alpha_s)(1-\sigma_t^{\mathrm{ReMDM}})}\right)\vx_s + \frac{\sigma_t^{\mathrm{ReMDM}}\alpha_s}{(1-\alpha_s)(1-\sigma_t^{\mathrm{ReMDM}})}\vx_0, \label{eq:bwdpc} \\
    \vx_{s_c} &= (1-\sigma_t^{\mathrm{ReMDM}})\vx_{s_b} + \sigma_t^{\mathrm{ReMDM}}\vpi. \label{eq:fwdpc}
\end{align}
\subsection{Proofs of Main Theorems}
\begin{proof}[Proof of Lemma~\ref{thm:contraction} (Forward Process Contraction)]
\textbf{\newline Contraction of masking corruption.}
    
    We use the following characterization of the SDPI constant for the KL divergence in terms of mutual information:
    \begin{align*}
        \eta_{\mathrm{KL}}(p_s,q_{t\mid s}) &= \sup_{\substack{P_{U\mid X_s}:U\to X_s\to X_t\\
        I(U;X_s)>0}}\frac{I(U;X_t)}{I(U;X_s)},
    \end{align*}
    where we denote $X_t \sim p_t$.
    
    % Masking
    For masking diffusion, we note that for $s = 0$ and $p_s(M) = 0$ we can simply use the proof for the binary erasure channel $S = \{0,1,M\}$ outlined in~\citep{PolyanskiyWu} and apply it to the $q$-ary case. However, because the input distribution support can include the masking symbol in general, we present a modified argument that applies to any arbitrary time $s$.
    
    Suppose we define an auxiliary random variable $B \sim \text{Ber}\left(1 - \frac{\alpha_t}{\alpha_s}\right)$ such that $B \independent (X_s,U)$. Further suppose we define 
    \begin{align*}
        X_t &= \begin{cases}
            X_s &\text{if } B = 0 \\
            M &\text{if } B = 1
        \end{cases}
    \end{align*}
    By the chain rule:
    \begin{align*}
        I(U;X_t) &\leq I(U;X_t,B) = I(U;B) + I(U;X_t\mid B) = I(U;X_t\mid B),
    \end{align*}
    where $I(U;B) = 0$ since $U\independent B$.
    
    We further find that
    \begin{align*}
        I(U;X_t\mid B) = \left(1-\frac{\alpha_t}{\alpha_s}\right)I(U;X_t\mid B=1) + \frac{\alpha_t}{\alpha_s}I(U;X_t\mid B=0) = \frac{\alpha_t}{\alpha_s}I(U;X_s).
    \end{align*}
    $I(U;X_t\mid B=1) = 0$ since $X_t = M$ deterministically, and $I(U;X_t\mid B=0) = I(U;X_s)$ since $X_t$ is unchanged from $X_s$ and $B \independent (X_s,U)$. Since this holds for all such $P_{U\mid X_s}$, we have $\eta_{\mathrm{KL}}(p_s,q_{t\mid s}^M) \leq \frac{\alpha_t}{\alpha_s}$.

    To establish equality, we construct a binary random variable $U$ such that $I(U;B\mid X_t) = 0$ so $I(U;X_t) = I(U;X_t,B)$ via the chain rule. 
    
    In addition, we assume that $p_0$ is not a point mass and is supported entirely on non-masking states, and that there exist distinct non-masking states $a$ and $b$ such that $p_s(a),p_s(b)>0$ whenever $\alpha_s>0$. 

    We then define $U$ according to
    \begin{align*}
        P(U = 1 \mid X_s = x) &=
        \begin{cases}
            1 &\text{if } x = a \\
            0 &\text{if } x \notin \{a,M\} \\
            P(X_s = a \mid X_s \neq M) &\text{if } x = M.
        \end{cases}
    \end{align*}
    By construction, $U \not\independent X_s$ and $I(U;X_s) > 0$ since for $X_s = a,b$ we have
    \begin{align*}
        P(U = 1 \mid X_s = a) &= 1 \\
        P(U = 1 \mid X_s = b) &= 0.
    \end{align*}

    Furthermore, we know that
    \begin{align*}
        P(U = 1 \mid X_s = M) &= P(X_s = a \mid X_s \neq M) \\
        P(U = 1 \mid X_s \neq M) &= P(X_s = a \mid X_s \neq M), 
    \end{align*}
    so $U \independent \mathbf{1}_{X_s \neq M}$.

    For $X_t = x_t \neq M$, we know that $B = 0$ deterministically. Hence $U \independent B \mid X_t = x_t$ in this case.

    For $X_t = x_t = M$, we first consider $B = 0$, which implies $X_t = X_s = M$. In this case
    \begin{align*}
        P(U = u \mid B = 0, X_t = M) &= P(U = u \mid B = 0, X_s = M) \\
        &= P(U = u \mid X_s = M) \\
        &= P(U = u),
    \end{align*}
    where in the second equality we used $B\independent (U,X_s)$ and the third equality we used $U \independent \mathbf{1}_{X_s \neq M}$.

    We then consider $B = 1$, where $X_t = M$ deterministically:
    \begin{align*}
        P(U = u \mid B = 1, X_t = M) &= P(U = u \mid B = 1) \\
        &= P(U = u),
    \end{align*}
    where in the second equality we used $B\independent U$.
    
    Hence, $U \independent B \mid X_t$, and $I(U;B\mid X_t) = 0$. This implies that the upper bound is attained and hence $\eta_{\mathrm{KL}}(p_s, q_{t\mid s}^M) = \frac{\alpha_t}{\alpha_s}$.

    \textbf{Contraction of uniform corruption.}
    
    % Uniform
    For uniform diffusion, deriving the SDPI constant is generally difficult, and only upper bounds are known for the general case. For $S = 2$, an upper bound is $\eta_{\mathrm{KL}}(q_{t\mid s}) = \sup_{p_s}\eta_{\mathrm{KL}}(p_s,q_{t\mid s}) = (\frac{\alpha_t}{\alpha_s})^2$, the proof of which we defer to~\cite{PolyanskiyWu}. An upper bound for the general case is outlined in Proposition 38 of~\citep{Potts}, given by $\eta_{\mathrm{KL}}(p_s,q_{t\mid s}) \leq \eta_{\mathrm{KL}}(q_{t\mid s}) \leq \frac{S(\frac{\alpha_t}{\alpha_s})^2}{(S-2)\frac{\alpha_t}{\alpha_s}+2}$. For large $S$, the upper bound approaches the worst-case contraction $\frac{\alpha_t}{\alpha_s}$. Intuitively, under a larger number of states $S$, there is more room for distributions to spread out without overlapping, whereas for small $S$, mixing tends to create more overlap in distribution. 
    
    % Proof of sharpness
    \textbf{Masking corruption attains the worst contraction.}
    
    Following~\citep{PolyanskiyWu}, we have the bounds $\eta_{\mathrm{KL}}(p_s,q_{t\mid s})\leq\eta_{\mathrm{TV}}(q_{t\mid s})$. In particular, suppose we let $\vx_t$ and $\hat{\vx}_t$ be the probability vectors representing the ground-truth and perturbed distributions, respectively. The forward process can be expressed as the update $\vx_t = \frac{\alpha_t}{\alpha_s}\vx_s + (1-\frac{\alpha_t}{\alpha_s})\vpi$.

    We can then analyze the contraction in TV distance directly:
    \begin{align*}
        D_{\mathrm{TV}}(\hat{\vx}_t,\vx_t) &= \frac{1}{2}\Vert\hat{\vx}_t - \vx_t\Vert_1 = \frac{1}{2}\frac{\alpha_t}{\alpha_s}\Vert\hat{\vx}_s - \vx_s\Vert_1 = \frac{\alpha_t}{\alpha_s}D_{\mathrm{TV}}(\hat{\vx}_s,\vx_s).
    \end{align*}
    
    Since the above holds for arbitrary input distributions, this implies that $\eta_{\mathrm{TV}}(q_{t\mid s}) = \frac{\alpha_t}{\alpha_s}$ and that TV distance always contracts by exactly $\frac{\alpha_t}{\alpha_s}$. Putting the above together, we conclude that $\eta_{\mathrm{KL}}(p_s,q_{t\mid s}^M)$ indeed achieves the upper bound.
\end{proof}

\begin{proof}[Proof of Lemma~\ref{thm:errorcorrection} (Reverse Process Contraction)]
\textbf{\newline Contraction of DDIM / DPF.}

We begin by analyzing the contraction in Equation (\ref{eq:DDIM}), given by $\vx_s = \sigma_{s\mid t}\vx_t + (1-\sigma_{s\mid t})\vx_0(\vx_t)$. To further simplify analysis, we can express the posterior as
\begin{align*}
    p_{0\mid t}(x_0\mid x_t) &= \frac{p_{t\mid 0}(x_t\mid x_0)p_0(x_0)}{p_t(x_t)} \\
    &= \frac{(\alpha_t\delta_{x_0}(x_t) + (1-\alpha_t)\pi(x_t))p_0(x_0)}{\alpha_tp_0(x_t)+(1-\alpha_t)\pi(x_t)} \\
    &= \left(\frac{\alpha_tp_0(x_t)}{\alpha_tp_0(x_t)+(1-\alpha_t)\pi(x_t)}\right)\delta_{x_t}(x_0) + \left(\frac{(1-\alpha_t)\pi(x_t)}{\alpha_tp_0(x_t)+(1-\alpha_t)\pi(x_t)}\right)p_0(x_0).
\end{align*}

In vector notation, we have $\vx_0(\vx_t) = \vf_t\odot\vx_t+\langle 1-\vf_t, \vx_t\rangle\vx_0$, where $(\vf_t)_i := \frac{\alpha_t(\vx_0)_i}{\alpha_t(\vx_0)_i+(1-\alpha_t)\vpi_i}$.

Using joint convexity, the KL divergence contraction
\begin{align*}
    D_{\mathrm{KL}}(\hat{\vx}_s \Vert \vx_s) &= D_{\mathrm{KL}}(\sigma_{s\mid t}\hat{\vx}_t + (1-\sigma_{s\mid t})\vx_0(\hat{\vx}_t) \Vert \sigma_{s\mid t}\vx_t + (1-\sigma_{s\mid t})\vx_0(\vx_t)) \\
    &\leq \sigma_{s\mid t}D_{\mathrm{KL}}(\hat{\vx}_t\Vert \vx_t) + (1-\sigma_{s\mid t})D_{\mathrm{KL}}(\vx_0(\hat{\vx}_t) \Vert \vx_0(\vx_t)) \\
    &\leq (\sigma_{s\mid t} + (1-\sigma_{s\mid t})\eta_{\mathrm{KL}}(p_t,p_{0\mid t}))D_{\mathrm{KL}}(\hat{\vx}_t\Vert \vx_t).
\end{align*}

We can find an upper bound on $\eta_{\mathrm{KL}}(p_t,p_{0\mid t})$ by analyzing the contraction in TV distance. To do so, we use Dobrushin's characterization:
\begin{align*}
    \eta_{\mathrm{TV}}(p_{0\mid t}) &= \sup_{i\neq j}D_{\mathrm{TV}}(\vx_0(\ve_i),\vx_0(\ve_j)) \\
    &= \sup_{i\neq j}D_{\mathrm{TV}}(\vf_t\odot\ve_i + \langle 1-\vf_t,\ve_i\rangle\vx_0,\vf_t\odot\ve_j + \langle 1-\vf_t,\ve_j\rangle\vx_0) \\
    &= \sup_{i\neq j}\frac{1}{2}\Vert((\vf_t)_i\ve_i + (1-(\vf_t)_i)\vx_0)-((\vf_t)_j\ve_j + (1-(\vf_t)_j)\vx_0)\Vert_1 \\
    &= \sup_{i\neq j}\frac{1}{2}\Vert((\vf_t)_i\ve_i + (\Vert\vf_t\Vert_{\infty}-(\vf_t)_i)\vx_0)-((\vf_t)_j\ve_j + (\Vert\vf_t\Vert_{\infty}-(\vf_t)_j)\vx_0)\Vert_1 \\
    &\leq \sup_{i\neq j}\frac{1}{2}\Vert(\vf_t)_i\ve_i + (\Vert\vf_t\Vert_{\infty}-(\vf_t)_i)\vx_0\Vert_1+\frac{1}{2}\Vert(\vf_t)_j\ve_j + (\Vert\vf_t\Vert_{\infty}-(\vf_t)_j)\vx_0\Vert_1 \\
    &= \sup_{i\neq j}\frac{(\vf_t)_i+\Vert\vf_t\Vert_{\infty}-(\vf_t)_i}{2} + \frac{(\vf_t)_j+\Vert\vf_t\Vert_{\infty}-(\vf_t)_j}{2} \\
    &= \Vert \vf_t\Vert_{\infty}.
\end{align*}
From the above we have that $\eta_{\mathrm{KL}}(p_t,p_{0\mid t})\leq\eta_{\mathrm{TV}}(p_{0\mid t}) \leq \Vert\vf_t\Vert_{\infty}$, where
\begin{align*}
    \Vert\vf_t\Vert_{\infty} = \max_{x_t\in \mathcal{S}}\frac{\alpha_tp_0(x_t)}{\alpha_tp_0(x_t)+(1-\alpha_t)\pi(x_t)}.
\end{align*}

% Sharpness
In the case of uniform diffusion we know that $\pi(x_t) > 0$ for all states $x_t \in \mathcal{S}$, so $\Vert\vf_t\Vert_{\infty}<1$ and $\eta_{\mathrm{KL}}(q_{s\mid t}^{\mathrm{DPF},U}) < 1$ for all times $t \in (0,1)$. In particular, we find that contraction is tightest when $\alpha_t \approx 0$ near the noise distribution and weakest when $\alpha_t \approx 1$ near the data distribution.

For masking diffusion, $\pi(x_t) = 0$ if $x_t$ is not a masking state, so $\Vert\vf_t\Vert_{\infty}=1$. Since $\eta_{\mathrm{KL}}(p_t,p_{0\mid t}) \leq 1$, this suggests that masking diffusion fails to contract in the worst case. 

We show that this bound is tight in the channel-wide sense. If we choose any two distinct input distributions supported entirely on non-masking states with finite KL divergence, the posterior preserves their KL divergence, implying that $\eta_{\mathrm{KL}}(p_{0\mid t})=\eta_{\mathrm{KL}}(q_{s\mid t}^{\mathrm{DPF},M})=1$.

In general, the upper bound $\eta_{\mathrm{KL}}(p_t,q_{s\mid t}^{\mathrm{DPF}})\leq \sigma_{s\mid t} + (1-\sigma_{s\mid t})\eta_{\mathrm{KL}}(p_t,p_{0\mid t})$ is not always tight, but tightening the bound may require stronger assumptions such as independence $\vx_0(\vx_t) = \vx_0$, which is only generally true at the noise distribution $t = 1$. In this case, when $\vx_0 = \ve_i$ for some state $i\in\mathcal{S}$, the transition update is a masking process and the channel-wide upper bound is attained at $\sigma_{s\mid t}$. When $\vx_0 = \frac{1}{S}\mathbf{1}$ the transition is a uniform process and the channel-wide upper bound can be sharpened to $\frac{S(\sigma_{s\mid t})^2}{(S-2)\sigma_{s\mid t}+2}$.

% Stochasticity
\textbf{Contraction under stochasticity.}
Under the stochasticity schedule $\nu_t$, we consider the following marginally consistent extension of the conditional DPF Euler update:
\begin{equation}
    \label{eq:stochasticityupdate}
    \vx_s = \left(1+\frac{(\alpha_t-\alpha_s)(\nu_t+\alpha_t)}{\alpha_t(1-\alpha_t)}\right)\vx_t+\frac{(1+\nu_t)(\alpha_s-\alpha_t)}{1-\alpha_t}\vx_0 + \frac{\nu_t(\alpha_s-\alpha_t)}{\alpha_t}\vpi,
\end{equation}
which corresponds to $\sigma_t^{\mathrm{ReMDM}} = \frac{\nu_t(\alpha_s-\alpha_t)}{\alpha_t}$~\citep{ReMDM}. While in principle $\nu_t$ can be driven arbitrarily large, the above update requires that the interval $[s,t]$ be chosen such that $\sigma_t^{\mathrm{ReMDM}} \leq \min\{1,\frac{1-\alpha_s}{\alpha_t}\}$ to ensure that the update samples from a valid probability distribution.

We use the fact that the step can be decomposed into a DDIM update from time $t$ to $s$, followed by a corrector step that preserves the marginal at time $s$ in Equation (\ref{eq:pc}). The corrector step itself decomposes into a DDIM step and a forward step in Equations (\ref{eq:bwdpc},\ref{eq:fwdpc}). Let $q^{b}_{s\mid t}$ denote the DPF step followed by the backward corrector step (\ref{eq:bwdpc}), and $q^{f}_{t\mid s}$ the forward step (\ref{eq:fwdpc}).

Using the fact that under compositions of Markov chains, SDPI constants are upper-bounded multiplicatively, and hence
\begin{align*}
    \eta_{\mathrm{KL}}(p_t,q_{s\mid t}^{\nu_t}) &\leq \eta_{\mathrm{KL}}(p_t,q_{s\mid t}^{b})\eta_{\mathrm{KL}}(q_{t\mid s}^{f}). 
\end{align*}
Since the backward step reuses the posterior prediction made at time $t$, joint convexity gives the upper bounds
\begin{align*}
    \eta_{\mathrm{KL}}(p_t,q_{s\mid t}^{b}) &\leq \left(1-\left(\frac{\frac{\nu_t(\alpha_s-\alpha_t)}{\alpha_t}}{1-\frac{\nu_t(\alpha_s-\alpha_t)}{\alpha_t}}\right)\left(\frac{\alpha_s}{1-\alpha_s}\right)\right)\eta_{\mathrm{KL}}(p_t,q_{s\mid t}^{\mathrm{DPF}}) \\
    &+ \left(\frac{\frac{\nu_t(\alpha_s-\alpha_t)}{\alpha_t}}{1-\frac{\nu_t(\alpha_s-\alpha_t)}{\alpha_t}}\right)\left(\frac{\alpha_s}{1-\alpha_s}\right)\eta_{\mathrm{KL}}(p_t,p_{0\mid t})  \\
    &=:
    \sigma_{s\mid t}^{b}\eta_{\mathrm{KL}}(p_t,q_{s\mid t}^{\mathrm{DPF}}) + (1-\sigma_{s\mid t}^{b})\eta_{\mathrm{KL}}(p_t,p_{0\mid t}), \\
    \eta_{\mathrm{KL}}(q_{t\mid s}^{f}) &\leq \left(1-\frac{\nu_t(\alpha_s-\alpha_t)}{\alpha_t}\right) \\
    &=: \sigma_{t\mid s}^{f} 
\end{align*}

Substituting the previously derived upper bound for the DPF contraction coefficient, the resulting contraction coefficient bound is minimized at \(\alpha_s\eta_{\mathrm{KL}}(p_t,p_{0\mid t})\) by setting $\sigma_t^{\mathrm{ReMDM}} = \frac{\nu_t(\alpha_s-\alpha_t)}{\alpha_t} = 1 - \alpha_s$, where $1-\alpha_s \leq \min\{1,\frac{1-\alpha_s}{\alpha_t}\}$ is always satisfied. This corresponds to a stochasticity schedule of $\nu_t = \frac{\alpha_t(1-\alpha_s)}{\alpha_s-\alpha_t}$. In light of Equation (\ref{eq:DDIM}), the update reduces to $\alpha_s\vx_0(\vx_t) + (1-\alpha_s)\vpi$. That is, samples rely entirely on the estimated posterior at time step $t$. However, this is not necessarily useful in practice due to the additional errors in the posterior estimate.
\end{proof}

\begin{proof}[Proof of Theorem~\ref{thm:errorbound} (Contraction Under Model Error)]
\textbf{\newline Contraction of DDIM / DPF.}
    
    We first focus on showing the DPF sampling bound, where the per-step marginal error bound follows from the triangle inequality:
    \begin{align*}
        D_{\mathrm{TV}}(p_{t_{k+1}}^{\theta},p_{t_{k+1}}) &= D_{\mathrm{TV}}(\vx_{t_{k+1}}^{\theta},\vx_{t_{k+1}}) \\
        &\leq \sigma_{t_{k+1}\mid t_k}D_{\mathrm{TV}}(\vx_{t_k}^{\theta},\vx_{t_{k}}) + (1-\sigma_{t_{k+1}\mid t_k})D_{\mathrm{TV}}(\vx_0^{\theta}(\vx_{t_k}^{\theta}),\vx_0(\vx_{t_k})) 
    \end{align*} 
    and
    \begin{align*}
        D_{\mathrm{TV}}(\vx_0^{\theta}(\vx_{t_k}^{\theta}),\vx_0(\vx_{t_k})) &\leq D_{\mathrm{TV}}(\vx_0^{\theta}(\vx_{t_k}^{\theta}),\vx_0(\vx_{t_k}^{\theta})) + D_{\mathrm{TV}}(\vx_0(\vx_{t_k}^{\theta}),\vx_0(\vx_{t_k})) \\
        &\leq D_{\mathrm{TV}}(p_{0\mid t_k}^{\theta},p_{0\mid t_k}\mid p_{t_k}^{\theta}) + \eta_{\mathrm{TV}}(p_{0\mid t_k})D_{\mathrm{TV}}(\vx_{t_k}^{\theta},\vx_{t_k}) \\
        &= \epsilon_{k} + \eta_{\mathrm{TV}}(p_{0\mid t_k})D_{\mathrm{TV}}(p_{t_k}^{\theta},p_{t_k}).
    \end{align*}

    Putting the above together, we have
    \begin{align*}
        D_{\mathrm{TV}}(p_{t_{k+1}}^{\theta},p_{t_{k+1}}) &\leq (\sigma_{t_{k+1}\mid t_k} + (1-\sigma_{t_{k+1}\mid t_k})\eta_{\mathrm{TV}}(p_{0\mid t_k}))D_{\mathrm{TV}}(p_{t_k}^{\theta},p_{t_k}) + (1-\sigma_{t_{k+1}\mid t_k})\epsilon_k.
    \end{align*}

    \textbf{Contraction under stochasticity.}

    Under the stochasticity schedule $\nu_t$, we treat the step updates as a composition, which in the proof of Lemma~\ref{thm:errorcorrection} we showed to be true in the absence of discretization errors. We then apply the triangle inequality in succession to separate the model errors:
    \begin{align*}
    D_{\mathrm{TV}}(p_{t_k,b}^{\theta},p_{t_k,b})
    &\leq (\sigma_{t_{k+1}\mid t_k}
    + (1-\sigma_{t_{k+1}\mid t_k})\eta_{\mathrm{TV}}(p_{0\mid t_k}))\,D_{\mathrm{TV}}(p_{t_k}^{\theta},p_{t_k}) + (1-\sigma_{t_{k+1}\mid t_k})\epsilon_k \\
    &:= A_{k,d}D_{\mathrm{TV}}(p_{t_k}^{\theta},p_{t_k})+B_{k,d}\epsilon_k \\
    D_{\mathrm{TV}}(p_{t_k,f}^{\theta},p_{t_k,f})
    &\leq \sigma_{t_{k+1}|t_k}^b\,D_{\mathrm{TV}}(p_{t_k,b}^{\theta},p_{t_k,b})
    + (1-\sigma_{t_{k+1}|t_k}^b)\,\big(\eta_{\mathrm{TV}}(p_{0\mid t_k})\,D_{\mathrm{TV}}(p_{t_k}^{\theta},p_{t_k}) + \epsilon_k\big) \\
    &:= A_{k,b}^1D_{\mathrm{TV}}(p_{t_k,b}^{\theta},p_{t_k,b})+A_{k,b}^2D_{\mathrm{TV}}(p_{t_k}^{\theta},p_{t_k})+B_{k,b}\epsilon_k \\
    D_{\mathrm{TV}}(p_{t_{k+1}}^{\theta},p_{t_{k+1}})
    &\leq \sigma_{t_{k+1}|t_k}^f\,D_{\mathrm{TV}}(p_{t_k,f}^{\theta},p_{t_k,f}) \\
    &:= A_{k,f}D_{\mathrm{TV}}(p_{t_k,f}^{\theta},p_{t_k,f}),
    \end{align*}
    where $\sigma_{t_{k+1}\mid t_k}^{b}$ and $\sigma_{t_{k+1}\mid t_k}^{f}$ are coefficients associated with the corrector steps, defined in the proof of Lemma~\ref{thm:errorcorrection}.

    The overall update reduces to
    \begin{align*}
        D_{\mathrm{TV}}(p_{t_{k+1}}^{\theta},p_{t_{k+1}}) &\leq (A_{k,d}A_{k,b}^1A_{k,f} + A_{k,b}^2A_{k,f})D_{\mathrm{TV}}(p_{t_k}^{\theta},p_{t_k}) + (A_{k,f}A_{k,b}^1B_{k,d} + A_{k,f}B_{k,b})\epsilon_{k}. 
    \end{align*}

    If $\nu_t = \frac{\alpha_t(1-\alpha_s)}{\alpha_s-\alpha_t}$, the update each step is $\vx_{t_{k+1}} = \alpha_{t_{k+1}}\vx_0(\vx_{t_k}) + (1-\alpha_{t_{k+1}})\vpi$. The inequality then reduces to
    \begin{align*}
        D_{\mathrm{TV}}(p_{t_{k+1}}^{\theta},p_{t_{k+1}}) &\leq \alpha_{t_{k+1}}\eta_{\mathrm{TV}}(p_{0\mid t_k})D_{\mathrm{TV}}(p_{t_k}^{\theta},p_{t_k}) + \alpha_{t_{k+1}}\epsilon_k. 
    \end{align*}

    In particular, we know that $\alpha_{t_{k+1}} \geq (1-\sigma_{t_{k+1}\mid t_k})$, where equality holds if and only if $\alpha_{t_k} = 0$ or $\alpha_{t_{k+1}} = 1$. This suggests that the stochasticity schedule $\nu_t$ uses updates that rely on the posterior $\vx_0(\vx_{t_k})$ at least as much as that of the DPF updates, resulting in a larger $\epsilon_k$ coefficient. This establishes that $B_k^{\nu_t} \geq B_k^{\mathrm{DPF}}$. 
    
    On the other hand, since
    $\eta_{\mathrm{TV}}(p_{0\mid t_k})\in[0,1]$,
    we have
    \begin{align*}
    A_k^{\mathrm{DPF}}-A_k^{\nu_t}
    &=
    \sigma_{t_{k+1}\mid t_k}
    \left(1-\eta_{\mathrm{TV}}(p_{0\mid t_k})\right) +
    (1-\alpha_{t_{k+1}})
    \eta_{\mathrm{TV}}(p_{0\mid t_k})
    \geq 0.
    \end{align*}
    Hence, $A_k^{\nu_t}\leq A_k^{\mathrm{DPF}}$ regardless
    of the contraction coefficient of the posterior kernel.
    In particular, even when the posterior kernel is not
    strictly contractive, i.e.,
    $\eta_{\mathrm{TV}}(p_{0\mid t_k})=1$,
    the stochasticity schedule ensures a contraction
    coefficient of at most $\alpha_{t_{k+1}}$.

    In both cases, this reduces to
    \begin{align*}
        D_{\mathrm{TV}}(p_{t_{k+1}}^{\theta}, p_{t_{k+1}}) \leq A_kD_{\mathrm{TV}}(p_{t_k}^{\theta}, p_{t_k}) + B_k\epsilon_k.
    \end{align*}
    for some coefficients $A_k$ and $B_k$.

    The total error bound over sampling $K$ steps proceeds via induction by successively applying upper bounds on the per-step marginal error:
    \begin{align*}
        D_{\mathrm{TV}}(p^\theta_0,p_0) &\le \left(\prod_{k=1}^{K}A_k\right) D_{\mathrm{TV}}(p_{t_1}^{\theta}, p_{t_1})+ \sum_{k=1}^K\left(\prod_{j=k+1}^K A_j\right)B_k \epsilon_k = \sum_{k=1}^K\left(\prod_{j=k+1}^K A_j\right)B_k \epsilon_k,
    \end{align*}
    where $p_{t_1}^{\theta} = p_{t_1}$ since the noise distribution at $t = 1$ is always sampled exactly.

    \textbf{DCRS Error Bound.}
    
    Under one restart of DCRS without any churning steps, the sampling bound for DPF carries over, except the forward process contracts by a factor of $\frac{\alpha_{t_{\mathrm{max}}}}{\alpha_{t_{\mathrm{min}}}}$. Suppose that $t_{\mathrm{min}} = t_{k_{\mathrm{min}}}$ and $t_{\mathrm{max}} = t_{k_{\mathrm{max}}}$. Here, we let $\epsilon_k$ denote an upper bound on the posterior estimation error at time $t_k$ across both the original and restarted reverse trajectories.
    \begin{align*}
        D_{\mathrm{TV}}(p^\theta_0,p_0) &\le \frac{\alpha_{t_{\mathrm{max}}}}{\alpha_{t_{\mathrm{min}}}}\left(\prod_{j=k_{\mathrm{max}}}^K A_j\right)\sum_{k=1}^{k_{\mathrm{min}}-1}\left(\prod_{j=k+1}^{k_{\mathrm{min}}-1} A_j\right)B_k \epsilon_k + \sum_{k=k_{\mathrm{max}}}^{K}\left(\prod_{j=k+1}^K A_j\right)B_k \epsilon_k.
    \end{align*} 
\end{proof}

\subsection{Extension to higher dimensions}
\begin{theorem}[KL Contraction Under Model Error]
\label{thm:global_tv_kl}
Let $1=t_1>\cdots>t_{K+1}=0$ be a time
discretization, and suppose the marginal distributions
evolve according to
\begin{equation}
    p_{t_{k+1}}
    =p_{t_k}p_{t_{k+1}\mid t_k},
    \qquad
    p_{t_{k+1}}^\theta
    =p_{t_k}^\theta p_{t_{k+1}\mid t_k}^\theta,
    \nonumber
\end{equation}

Suppose the exact transition admits the posterior
representation
\begin{equation*}
    p_{t_{k+1}\mid t_k}(y\mid x)
    =
    \sum_{\vx_0}
    p_{0\mid t_k}(\vx_0\mid x)
    p_{t_{k+1}\mid t_k,0}(y\mid x,\vx_0).
\end{equation*}

Define the parallel transition using the product
of the exact posterior marginals,
\begin{align*}
    p_{t_{k+1}\mid t_k}^{\mathrm{par}}(y\mid x)
    &=
    \sum_{\vx_0}
    \prod_{i=1}^D p_{0\mid t_k}^{(i)}
        (x_0^i\mid x) p_{t_{k+1}\mid t_k,0}
        (y\mid x,\vx_0),
\end{align*}
and let the learned parallel transition be
\begin{align*}
    p_{t_{k+1}\mid t_k}^{\theta}(y\mid x)
    &=
    \sum_{\vx_0}
    \prod_{i=1}^D p_{0\mid t_k}^{\theta,(i)}
        (x_0^i\mid x)p_{t_{k+1}\mid t_k,0}
        (y\mid x,\vx_0).
\end{align*}

Assume the following conditions hold:

\begin{enumerate}
    \item \textbf{Posterior model error.}
    For every $k=1,\ldots,K$,
    \begin{align*}
        &\mathbb E_{x\sim p_{t_k}^\theta}
        \left[
        \sum_{i=1}^D
        D_{\mathrm{KL}}\left(
            p_{0\mid t_k}^{\theta,(i)}(\cdot\mid x)
            \middle\|
            p_{0\mid t_k}^{(i)}(\cdot\mid x)
        \right)
        \right]
        \leq \epsilon_k^2.
    \end{align*}

    \item \textbf{Parallel decoding error.}
    Define
    \begin{align*}
        \delta_k^2
        &=
        \mathbb E_{x\sim p_{t_k}^\theta}
        \Bigg[
        D_{\mathrm{KL}}\Bigg(
            p_{0\mid t_k}(\cdot\mid x)
            \Bigg\|
            \prod_{i=1}^D
            p_{0\mid t_k}^{(i)}(\cdot\mid x)
        \Bigg)
        \Bigg].
    \end{align*}

    \item \textbf{Conditional decoder contraction.}
    For every $x$, the conditional transition
    $p_{t_{k+1}\mid t_k,0}(\cdot\mid x,\vx_0)$,
    viewed as a channel from $\vx_0$ to $y$,
    satisfies
    \begin{equation*}
        \sup_x\eta_{\mathrm{KL}}\left(
            p_{t_{k+1}\mid t_k,0}(\cdot\mid x,\cdot)
        \right)
        \leq \xi_k.
    \end{equation*}

    \item \textbf{Transition contraction.}
    Each exact transition satisfies
    \begin{equation*}
        \eta_{\mathrm{KL}}(p_{t_{k+1}\mid t_k}) \leq \eta_k.
    \end{equation*}
\end{enumerate}

Then the final sampling error satisfies
\begin{equation*}
\begin{aligned}
    D_{\mathrm{TV}}(p_{t_{K+1}}^\theta,p_{t_{K+1}})\leq
    \frac{1}{\sqrt{2}}
    \sum_{k=1}^{K}
    \sqrt{\xi_k}
    (\epsilon_k+\delta_k)
    \prod_{j=k+1}^{K}\sqrt{\eta_j}.
\end{aligned}
\end{equation*}
\end{theorem}

\textbf{Remark.} The previous dependence on stochasticity manifests through the contraction coefficient of the conditional decoder $\xi_k$, where the overall contraction of the conditional decoder and posterior manifests through $\eta_k$.

Contraction of the conditional decoder in this case is now highly dependent on the input distributions and must be handled carefully. In particular, since the conditional decoder is a product of site-wise channels, we obtain the dimension-dependent bound
\[
    \eta_{\mathrm{KL}}(q_{t|s}^{\otimes D})
    \leq 1 - \big(1 - \eta_{\mathrm{KL}}(q_{t|s})\big)^D.
\]

This guarantees strict contraction whenever the one-dimensional channel is strictly contractive, but it also exposes a worst-case curse of dimensionality: the coefficient bound approaches one exponentially fast in dimension $D$. This behavior is unavoidable without assumptions on the input dependence structure; for example, highly coupled distributions supported on repeated-symbol sequences can mix extremely slowly under coordinate-wise corruption (e.g. AAAA, BBBB). Under additional assumptions that relax how tightly coupled the dimensions are, the exponent can be reduced depending on the maximum number of dimensions a particular coordinate influences or is influenced by \cite{SDPI}. In particular, under independent inputs, we recover stronger contraction
\[
    \eta_{\mathrm{KL}}(q_s^{\otimes D}, q_{t|s}^{\otimes D})
        = \eta_{\mathrm{KL}}(q_s, q_{t|s}).
\]

\begin{proof}
\textbf{Proof outline.}

The proof proceeds in three stages. First, a telescoping argument and the
triangle inequality in TV separate the error introduced at each step into a
posterior model error and a parallel decoding error (Step~1). Next, Pinsker's
inequality converts each term to KL divergence, where the decoder contraction
$\xi_k$ and the transition contraction $\prod_{j>k}\eta_j$ bound each error
source (Steps~2--3). Finally, summing the per-step contributions yields the
global bound (Step~4).

\textbf{Step 1: Telescoping decomposition.}

For $1\leq k\leq K+1$, define the hybrid distribution
\begin{equation*}
    \widetilde p_{t_{K+1}}^{(k)}
    =
    p_{t_k}^\theta p_{t_{K+1}\mid t_k},
\end{equation*}
where $p_{t_{K+1}\mid t_k}$ denotes the composition
of the exact transition kernels from $t_k$ to $t_{K+1}$, with the convention $p_{t_{K+1}\mid t_{K+1}} = I$.

Since $p_{t_1}^\theta=p_{t_1}$, $\widetilde p_{t_{K+1}}^{(1)}=p_{t_{K+1}}$ and $\widetilde p_{t_{K+1}}^{(K+1)}=p_{t_{K+1}}^\theta$.

By the triangle inequality,
\begin{align}
    D_{\mathrm{TV}}(p_{t_{K+1}}^\theta,p_{t_{K+1}})
    \leq
    \sum_{k=1}^{K}
    D_{\mathrm{TV}}\left(
        \widetilde p_{t_{K+1}}^{(k+1)},
        \widetilde p_{t_{K+1}}^{(k)}
    \right).
    \label{eq:tv_telescope_post}
\end{align}

For each $k$, we have
\begin{align*}
    \widetilde p_{t_{K+1}}^{(k+1)}
    &=
    p_{t_k}^\theta
    p_{t_{k+1}\mid t_k}^\theta
    p_{t_{K+1}\mid t_{k+1}},\qquad\widetilde p_{t_{K+1}}^{(k)}
    =
    p_{t_k}^\theta
    p_{t_{k+1}\mid t_k}
    p_{t_{K+1}\mid t_{k+1}}.
\end{align*}

Introduce the intermediate distribution
\begin{equation*}
    \overline p_{t_{K+1}}^{(k)}
    =
    p_{t_k}^\theta
    p_{t_{k+1}\mid t_k}^{\mathrm{par}}
    p_{t_{K+1}\mid t_{k+1}}.
\end{equation*}

Applying the triangle inequality again,
\begin{align}
    D_{\mathrm{TV}}\left(
        \widetilde p_{t_{K+1}}^{(k+1)},
        \widetilde p_{t_{K+1}}^{(k)}
    \right)
    \leq
    D_{\mathrm{TV}}\left(
        \widetilde p_{t_{K+1}}^{(k+1)},
        \overline p_{t_{K+1}}^{(k)}
    \right)
    +
    D_{\mathrm{TV}}\left(
        \overline p_{t_{K+1}}^{(k)},
        \widetilde p_{t_{K+1}}^{(k)}
    \right).
    \label{eq:tv_error_separation}
\end{align}

The first term represents posterior model error,
while the second represents parallel decoding error.

\medskip
\noindent
\textbf{Step 2: Bounding the posterior model error.}

By Pinsker's inequality,
\begin{align*}
    &D_{\mathrm{TV}}\left(
        \widetilde p_{t_{K+1}}^{(k+1)},
        \overline p_{t_{K+1}}^{(k)}
    \right)\leq
    \Bigg[
    \frac12D_{\mathrm{KL}}\Big(
        p_{t_k}^\theta
        p_{t_{k+1}\mid t_k}^{\theta}
        p_{t_{K+1}\mid t_{k+1}} \Big\|
        p_{t_k}^\theta
        p_{t_{k+1}\mid t_k}^{\mathrm{par}}
        p_{t_{K+1}\mid t_{k+1}}
    \Big)
    \Bigg]^{1/2}.
\end{align*}

By successive applications of KL-SDPI,
\begin{align}
    &D_{\mathrm{KL}}\Big(
        p_{t_k}^\theta
        p_{t_{k+1}\mid t_k}^{\theta}
        p_{t_{K+1}\mid t_{k+1}}
        \Big\|
        p_{t_k}^\theta
        p_{t_{k+1}\mid t_k}^{\mathrm{par}}
        p_{t_{K+1}\mid t_{k+1}}
    \Big)\nonumber \\
    &\leq
    \left(\prod_{j=k+1}^{K}\eta_j\right)
    D_{\mathrm{KL}}\left(
        p_{t_k}^\theta
        p_{t_{k+1}\mid t_k}^{\theta}
        \middle\|
        p_{t_k}^\theta
        p_{t_{k+1}\mid t_k}^{\mathrm{par}}
    \right).
    \label{eq:future_contraction_model}
\end{align}

By joint convexity of KL divergence,
\begin{align*}
    &D_{\mathrm{KL}}\left(
        p_{t_k}^\theta
        p_{t_{k+1}\mid t_k}^{\theta}
        \middle\|
        p_{t_k}^\theta
        p_{t_{k+1}\mid t_k}^{\mathrm{par}}
    \right) \leq
    \mathbb E_{x\sim p_{t_k}^\theta}
    \Big[
    D_{\mathrm{KL}}\big(
        p_{t_{k+1}\mid t_k}^{\theta}(x,\cdot)\big\|
        p_{t_{k+1}\mid t_k}^{\mathrm{par}}(x,\cdot)
    \big)
    \Big].
\end{align*}

For fixed $x$, both transitions are obtained
by applying the same conditional decoder
$p_{t_{k+1}\mid t_k,0}$ to different
posterior distributions.

By the conditional decoder KL-SDPI,
\begin{align*}
    &D_{\mathrm{KL}}\left(
        p_{t_{k+1}\mid t_k}^{\theta}(x,\cdot)
        \middle\|
        p_{t_{k+1}\mid t_k}^{\mathrm{par}}(x,\cdot)
    \right) \leq
    \xi_k
    D_{\mathrm{KL}}\left(
        \prod_{i=1}^D
        p_{0\mid t_k}^{\theta,(i)}(\cdot\mid x)
        \middle\|
        \prod_{i=1}^D
        p_{0\mid t_k}^{(i)}(\cdot\mid x)
    \right).
\end{align*}

Since KL divergence is additive over products,
\begin{align*}
    &D_{\mathrm{KL}}\left(
        \prod_{i=1}^D
        p_{0\mid t_k}^{\theta,(i)}(\cdot\mid x)
        \middle\|
        \prod_{i=1}^D
        p_{0\mid t_k}^{(i)}(\cdot\mid x)
    \right) =
    \sum_{i=1}^D
    D_{\mathrm{KL}}\left(
        p_{0\mid t_k}^{\theta,(i)}(\cdot\mid x)
        \middle\|
        p_{0\mid t_k}^{(i)}(\cdot\mid x)
    \right).
\end{align*}

Taking expectations and applying the posterior model-error assumption gives
\begin{equation*}
    D_{\mathrm{KL}}\left(
        p_{t_k}^\theta
        p_{t_{k+1}\mid t_k}^{\theta}
        \middle\|
        p_{t_k}^\theta
        p_{t_{k+1}\mid t_k}^{\mathrm{par}}
    \right)
    \leq \xi_k\epsilon_k^2.
\end{equation*}

Combining the preceding inequalities,
\begin{align}
    &D_{\mathrm{TV}}\left(
        \widetilde p_{t_{K+1}}^{(k+1)},
        \overline p_{t_{K+1}}^{(k)}
    \right) \leq
    \frac{1}{\sqrt2}
    \sqrt{\xi_k}\epsilon_k
    \prod_{j=k+1}^{K}\sqrt{\eta_j}.
    \label{eq:global_model_contribution}
\end{align}

\medskip
\noindent
\textbf{Step 3: Bounding the parallel decoding error.}

Similarly, by Pinsker's inequality,
\begin{align*}
    D_{\mathrm{TV}}\left(
        \overline p_{t_{K+1}}^{(k)},
        \widetilde p_{t_{K+1}}^{(k)}
    \right) \leq
    \Bigg[
    \frac12
    D_{\mathrm{KL}}\Big(
        p_{t_k}^\theta
        p_{t_{k+1}\mid t_k}
        p_{t_{K+1}\mid t_{k+1}} \Big\|
        p_{t_k}^\theta
        p_{t_{k+1}\mid t_k}^{\mathrm{par}}
        p_{t_{K+1}\mid t_{k+1}}
    \Big)
    \Bigg]^{1/2}.
\end{align*}

Notice that the direction of KL divergence
has been reversed, which is permitted because
total variation is symmetric.

Applying KL-SDPI to the subsequent exact
transitions,
\begin{align*}
    &D_{\mathrm{KL}}\Big(
        p_{t_k}^\theta
        p_{t_{k+1}\mid t_k}
        p_{t_{K+1}\mid t_{k+1}} \Big\|
        p_{t_k}^\theta
        p_{t_{k+1}\mid t_k}^{\mathrm{par}}
        p_{t_{K+1}\mid t_{k+1}}
    \Big) \leq
    \left(\prod_{j=k+1}^{K}\eta_j\right)
    D_{\mathrm{KL}}\left(
        p_{t_k}^\theta
        p_{t_{k+1}\mid t_k}
        \middle\|
        p_{t_k}^\theta
        p_{t_{k+1}\mid t_k}^{\mathrm{par}}
    \right).
\end{align*}

By joint convexity,
\begin{align*}
    &D_{\mathrm{KL}}\left(
        p_{t_k}^\theta
        p_{t_{k+1}\mid t_k}
        \middle\|
        p_{t_k}^\theta
        p_{t_{k+1}\mid t_k}^{\mathrm{par}}
    \right)\leq
    \mathbb E_{x\sim p_{t_k}^\theta}
    \Big[
    D_{\mathrm{KL}}\big(
        p_{t_{k+1}\mid t_k}(x,\cdot) \big\|
        p_{t_{k+1}\mid t_k}^{\mathrm{par}}(x,\cdot)
    \big)
    \Big].
\end{align*}

Both transitions again share the same
conditional decoder. Hence,
\begin{align*}
    &D_{\mathrm{KL}}\left(
        p_{t_{k+1}\mid t_k}(x,\cdot)
        \middle\|
        p_{t_{k+1}\mid t_k}^{\mathrm{par}}(x,\cdot)
    \right) \leq
    \xi_k
    D_{\mathrm{KL}}\left(
        p_{0\mid t_k}(\cdot\mid x)
        \middle\|
        \prod_{i=1}^D
        p_{0\mid t_k}^{(i)}(\cdot\mid x)
    \right).
\end{align*}

By the definition of $\delta_k$,
\begin{equation*}
    D_{\mathrm{KL}}\left(
        p_{t_k}^\theta
        p_{t_{k+1}\mid t_k}
        \middle\|
        p_{t_k}^\theta
        p_{t_{k+1}\mid t_k}^{\mathrm{par}}
    \right)
    \leq \xi_k\delta_k^2.
\end{equation*}

Consequently,
\begin{align}
    &D_{\mathrm{TV}}\left(
        \overline p_{t_{K+1}}^{(k)},
        \widetilde p_{t_{K+1}}^{(k)}
    \right) \leq
    \frac{1}{\sqrt2}
    \sqrt{\xi_k}\delta_k
    \prod_{j=k+1}^{K}\sqrt{\eta_j}.
    \label{eq:global_parallel_contribution}
\end{align}

\medskip
\noindent
\textbf{Step 4: Combining the inequalities.}

Substituting
\eqref{eq:global_model_contribution}
and \eqref{eq:global_parallel_contribution}
into \eqref{eq:tv_error_separation},
we obtain
\begin{align*}
    &D_{\mathrm{TV}}\left(
        \widetilde p_{t_{K+1}}^{(k+1)},
        \widetilde p_{t_{K+1}}^{(k)}
    \right) \leq
    \frac{1}{\sqrt2}
    \sqrt{\xi_k}
    (\epsilon_k+\delta_k)
    \prod_{j=k+1}^{K}\sqrt{\eta_j}.
\end{align*}

Finally, summing over $k$ using
\eqref{eq:tv_telescope_post} gives
\begin{align*}
    D_{\mathrm{TV}}(p_{t_{K+1}}^\theta,p_{t_{K+1}}) \leq
    \frac{1}{\sqrt2}
    \sum_{k=1}^{K}
    \sqrt{\xi_k}
    (\epsilon_k+\delta_k)
    \prod_{j=k+1}^{K}\sqrt{\eta_j},
\end{align*}
which completes the proof.
\end{proof}
\subsection{Contraction for General Corruption Processes}
The discretized-Gaussian corruption used for CIFAR10 and CelebA is not contained in the family of corruption processes analyzed by our theory. Nevertheless, the error-correction mechanism is not restricted to that family. For a general finite-state corruption channel $p_{t|0}$, we can evaluate its Dobrushin contraction coefficient, $\eta_{\mathrm{KL}}(p_{t|0}) \leq \eta_{\mathrm{TV}}(p_{t|0}) = \sup_{x,x'}D_{\mathrm{TV}}(p_{t|0}(\cdot|x),p_{t|0}(\cdot|x'))$,
which provides a computable upper bound on its KL contraction coefficient. We evaluated this quantity for the discretized-Gaussian channel and also estimated the KL and TV contraction ratios over 1,000 randomly sampled pairs of input distributions, whose maximum lower-bounds the true contraction coefficient. The results show strict contraction away from $t=0$:

\begin{table}[htbp]
\centering
\caption{Empirical KL and TV contraction (maximum ratio over 1,000 random
input pairs, a lower bound on the contraction coefficient) compared with
the Dobrushin TV coefficient of the discretized-Gaussian channel.}
\label{tab:contraction}
\begin{tabular}{cccc}
\toprule
$t$ & Empirical KL contraction & Empirical TV contraction & Discrete Gaussian Dobrushin TV \\
\midrule
0.0000 & 0.9999 & 0.9999 & 1.0000 \\
0.2500 & 0.0514 & 0.2793 & 0.9949 \\
0.5000 & 0.0218 & 0.1812 & 0.7243 \\
0.7500 & 0.0014 & 0.0466 & 0.1872 \\
1.0000 & 0.0000 & 0.0006 & 0.0025 \\
\bottomrule
\end{tabular}
\end{table}

Thus, the image corruption process exhibits the same qualitative contraction mechanism, although our closed-form comparison between masking and uniform corruption does not extend to it.

\section{Algorithms}\label{section:alg}

\subsection{EDM Discretization Scheme}
To improve sampling quality \citet{EDM} proposed a heuristic time discretization scheme that decreases the step sizes towards the end of the generation process. We also adopt this scheme with $\rho = 7$ when sampling with DCRS, where the discretization scheme is applied to the overall time interval $[0,1]$, as well as within the restarted interval $[t_{\text{min}},t_{\text{max}}]$. We present the exact mathematical formulation of the scheme in Algorithm \ref{alg:discretization_scheme}.
\begin{algorithm}[H]
   \caption{EDM\_Discretization\_Scheme$(t_{\text{min}},t_{\text{max}},N,\rho)$}
   \label{alg:discretization_scheme}
\begin{algorithmic}[1]
    \STATE \textbf{return} $\left\{\left(t_{\text{max}}^{\frac{1}{\rho}} + \frac{i}{N-1}\left(t_{\text{min}}^{\frac{1}{\rho}} - t_{\text{max}}^{\frac{1}{\rho}}\right)\right)^{\rho}\right\}_{i=0}^{N-1}$
\end{algorithmic}
\end{algorithm}

\subsection{Discrete Churn and Restart Sampling}

DCRS consists of running the reverse process through $\tau$-leaping and running the forward process through sampling from the closed-form categorical distribution defined by the transition matrix $q_{t_{\text{max}}\mid t_{\text{min}}}$. We present pseudocode for both restarting with the forward process and $\tau$-leaping in Algorithms~\ref{alg:restart_step} and~\ref{alg:poisson_step}, respectively.
\begin{algorithm}[H]
   \caption{Restart\_Step$(x_{t_\text{min}},q_{t_{\text{max}}\mid t_{\text{min}}})$}
   \label{alg:restart_step}
\begin{algorithmic}[1]
    \STATE $x_{t_{\max}} \sim \text{Cat}(x_{t_{\text{min}}}q_{t_{\text{max}}\mid t_{\text{min}}})$
    \STATE \textbf{return} $x_{t_{\text{max}}}$
\end{algorithmic}
\end{algorithm}

\begin{algorithm}[H]
   \caption{$\tau$-Leaping\_Step$(t_i,t_{i+1},x_{t_i},\hat{R}_t(\cdot))$}
   \label{alg:poisson_step}
\begin{algorithmic}[1]
    \FOR{$y \in B_1(x_{t_i})$}
        \STATE $P_{y,i} \sim \text{Poisson}((t_{i} - t_{i+1})\hat{R}_t(x_{t_i},y))$ 
    \ENDFOR
    \STATE $x_{t_{i+1}} = x_{t_i} + \sum_{y\in B_1(x_{t_i})}P_{y,i}(y - x_{t_i})$
    \STATE $x_{t_{i+1}} = \text{Clamp}\{x_{t_{i+1}},\text{min} = 1, \text{max} = S\}$
    \STATE \textbf{return} $x_{t_{i+1}}$
\end{algorithmic}
\end{algorithm}

We also explicitly define churning steps in Algorithm~\ref{alg:churn_step}, which injects a small amount of noise before every step from a solver of choice, such as $\tau$-leaping or trapezoidal steps:
\begin{algorithm}[H]
   \caption{Churn\_Step$(t_i,x_{t_i},\gamma)$}
   \label{alg:churn_step}
\begin{algorithmic}[1]
    \STATE $t_{\text{churn}} = (1 + \gamma)t_{i}$
    \STATE $x_{(1 + \gamma)t_{i}} \sim q_{{(1 + \gamma)t_{i}}\mid t_{i}}(\cdot \mid x_{t_{i}})$
    \STATE \textbf{return} $t_{\text{churn}},x_{t_{\text{churn}}}$
\end{algorithmic}
\end{algorithm}

\subsection{Higher-Order Solver}

As previously noted, $\tau$-leaping steps produce errors on the order of $O(\Delta t^2)$, rendering the algorithm first-order accurate. 
To obtain a second-order accurate solver, \citet{RKT} proposes using trapezoidal iterations inspired by SDE solvers. The iterations evaluate the reverse rate matrix both at the beginning of the interval and at an intermediate step. The step is then performed, taking a weighted combination of reverse rate matrix evaluations at both steps, and can be used as a plugin replacement for $\tau$-leaping. For full details of the trapezoidal solver, we refer to Algorithm \ref{alg:trapezoidal_step}.

\begin{algorithm}[H]\caption{Trapezoidal\_Step$(t_i,t_{i+1},x_{t_i},\hat{R}_t(\cdot),\theta)$}
   \label{alg:trapezoidal_step}
\begin{algorithmic}[1]
    \STATE $t_{\text{mid}} = \theta(t_{i+1}-t_i) + t_i$
    \STATE $w = \frac{1}{2\theta(1-\theta)}$
    \STATE $x_{t_{\text{mid}}} = \text{$\tau$-Leaping\_Step}(t_{i},t_{\text{mid}},x_{t_i},\hat{R}_t(x_{t_i}))$
    \STATE $x_{t_{i+1}} = \text{$\tau$-Leaping\_Step}(t_{\text{mid}},t_{i+1},x_{t_{\text{mid}}},\big(w\,\hat{R}_{t_{\text{mid}}}(x_{t_{\text{mid}}}) - (w-1)\,\hat{R}_{t_i}(x_{t_i})\big)_+)$
    \STATE \textbf{return} $x_{t_{i+1}}$
\end{algorithmic}
\end{algorithm}

\subsection{Discrete Churn and Restart Sampling Algorithm}
We present the full pseudocode for DCRS, where we denote the main hyperparameters of interest $N_{\text{Main}}$ and $N_{\text{Restart},j}$ to indicate the number of discretization steps used in the overall main interval $[t_{\text{stop}},1]$ and the $j$th restarted intervals $[t_{\text{min},j},t_{\text{max},j}]$, respectively, for up to $l$ different restarted intervals. We also introduce the hyperparameter $K_j$ to indicate the number of restarts performed for the $j$th restart interval. 

\begin{algorithm}[H]
   \caption{Discrete\_Churn\_and\_Restart\_Sampling\\
   $(\pi,\rho,\gamma,t_{\text{stop}},\hat{R}_t(\cdot),N_{\text{Main}},(N_{\text{Restart},j},K_j,t_{\text{min},j},t_{\text{max},j})_{j=1}^l)$}
   \label{alg:discrete_restart_sampling}
\begin{algorithmic}[1]
    \STATE $\{t_{i}\}_{i=0}^{N_{\text{Main}}} = \text{EDM\_Discretization\_Scheme}(t_{\text{stop}},1,N_{\text{Main}}+1,\rho)$
    \STATE Round $\{t_{\text{min},j}\}_{j=1}^l$ to the nearest neighbor in $\{t_i\}_{i=0}^{N_{\text{Main}}}$
    \STATE $x_{t_0} \sim \text{Cat}(\pi)$
    \FOR{$i = 0,\ldots,N_{\text{Main}}-1$}
        \STATE $x_{t_{i+1}} = \text{$\tau$-Leaping\_Step}(t_i,t_{i+1},x_{t_i},\hat{R}_{t_i}(\cdot))$
        \IF{$t_{i+1} = t_{\text{min},j}$ for some $j \in\{1,\ldots,l\}$}
            \STATE $\{s_{m}\}_{m=0}^{N_{\text{Restart},j}} = \text{EDM\_Discretization\_Scheme}(t_{\text{min},j},t_{\text{max},j},N_{\text{Restart},j}+1,\rho)$
            \STATE $x^{0} = x_{t_{i+1}}$
            \FOR{$k = 0,\ldots,K_j-1$}
                \STATE $y_{s_0} = \text{Restart\_Step}(x^{k},q_{t_{\text{max},j}\mid t_{\text{min},j}})$
                \FOR{$m = 0,\ldots,N_{\text{Restart},j}-1$}
                    \STATE $s_{\text{churn}},y_{s_{\text{churn}}} = \text{Churn\_Step}(s_m,y_{s_m},\gamma)$
                    \STATE $y_{s_{m+1}} = \text{Trapezoidal\_Step}(s_{\text{churn}},s_{m+1},y_{s_{\text{churn}}},\hat{R}_t(\cdot),\tfrac{1}{2})$
                \ENDFOR
                \STATE $x^{k+1} = y_{s_{N_{\text{Restart},j}}}$
            \ENDFOR
            \STATE $x_{t_{i+1}} = x^{K_j}$
        \ENDIF
    \ENDFOR
    \STATE $x_0 \sim \hat p_{0\mid t_{\text{stop}}}(\cdot\mid x_{t_{N_{\text{Main}}}})$ \hfill $\triangleright$ final denoising step
    \STATE \textbf{return} $x_0$
\end{algorithmic}
\end{algorithm}

%%%%%%%%%%%%%%%%%%%%%%%%%%%%%%%%%%%%%%%%%%%%%%%%%%%%%%%%%%%%%%%%%%%%%%%%%%%%%%%
%%%%%%%%%%%%%%%%%%%%%%%%%%%%%%%%%%%%%%%%%%%%%%%%%%%%%%%%%%%%%%%%%%%%%%%%%%%%%%%
\newpage
\section{Additional Experimental Details}
\subsection{Tables}
\begin{table}[htbp]
\centering
\begin{tabular}{lcc}
\hline
 & NFE & FID \\
\hline
DPF & 8 & 150.99 \\
 & 16 & 51.66 \\
 & 32 & 33.66 \\
 & 64 & 39.43 \\
 & 128 & 43.29 \\
 & 256 & 45.22 \\
 & 512 & 46.47 \\
\hline
$\tau$-Leaping & 8 & 366.74 \\
 & 16 & 330.90 \\
 & 32 & 316.35 \\
 & 64 & 151.22 \\
 & 128 & 21.73 \\
 & 256 & 7.11 \\
 & 512 & 6.09 \\
\hline
DCRS & 12 & 31.33 \\
 & 26 & 10.80 \\
 & 54 & 11.37 \\
 & 110 & 14.01 \\
 & 222 & 15.32 \\
 & 446 & 16.69 \\
\hline
\end{tabular}
\caption{CIFAR10}
\label{tab:cifar10}
\end{table}

\begin{table}[htbp]
\centering
\begin{tabular}{lcc}
\hline
 & NFE & FID \\
\hline
DPF & 8 & 98.96 \\
 & 16 & 45.22 \\
 & 32 & 30.61 \\
 & 64 & 33.55 \\
 & 128 & 36.39 \\
 & 256 & 38.15 \\
 & 512 & 39.07 \\
\hline
$\tau$-Leaping & 8 & 482.36 \\
 & 16 & 467.12 \\
 & 32 & 390.57 \\
 & 64 & 281.28 \\
 & 128 & 65.64 \\
 & 256 & 22.46 \\
 & 512 & 17.08 \\
\hline
DCRS & 12 & 53.52 \\
 & 26 & 20.44 \\
 & 54 & 15.90 \\
 & 110 & 19.71 \\
 & 222 & 23.19 \\
 & 446 & 25.33 \\
\hline
\end{tabular}
\caption{CelebA}
\label{tab:celeba}
\end{table}

\begin{table}[htbp]
\centering
\begin{tabular}{lcc}
\hline
 & NFE & $W_1$ Distance \\
\hline
DPF & 4 & 0.70 \\
 & 8 & 0.60 \\
 & 16 & 0.60 \\
 & 32 & 0.62 \\
 & 64 & 0.58 \\
 & 128 & 0.59 \\
 & 256 & 0.58 \\
 & 512 & 0.57 \\
\hline
$\tau$-Leaping & 4 & 2.60 \\
 & 8 & 2.56 \\
 & 16 & 2.48 \\
 & 32 & 2.45 \\
 & 64 & 2.02 \\
 & 128 & 1.00 \\
 & 256 & 0.42 \\
 & 512 & 0.41 \\
\hline
DCRS & 4 & 0.81 \\
 & 8 & 0.50 \\
 & 16 & 0.43 \\
 & 32 & 0.41 \\
 & 64 & 0.40 \\
 & 128 & 0.41 \\
 & 256 & 0.40 \\
 & 512 & 0.39 \\
\hline
\end{tabular}
\caption{Uniform Gaussians 8D}
\label{tab:uniformgaussians8D}
\end{table}

\newpage
\subsection{Sampling Details}
Previous works have noted that the discrete score function $\frac{p_t(y)}{p_t(x)}$ becomes poorly conditioned as $t$ approaches $0$. This is attributed to the data distribution behaving nearly deterministically conditioned on the other tokens, making the discrete score function essentially singular \citep{TLDR}. To mitigate this effect, iterations are typically stopped at some minimum time $t_{\text{stop}} = 10^{-3}$, which we also employ.

Because absorbing diffusion models tend not to fully unmask all tokens after sampling, we perform an additional model evaluation at $t_{\text{stop}}$ and sample from the categorical distribution $p_{t_{\text{stop}}}(x^d\mid x^{\setminus d})$.
\subsection{Toy Datasets}
\subsubsection{1D Dataset}
\label{section:1D}
Following~\citep{RKT}, we choose a ground-truth distribution over states uniformly at random from the simplex, where the number of states is set to $S = 15$.  We analytically derive the posterior $p_{0\mid t}$, where performance under model approximation errors is tested through random perturbations of the score function by a multiplicative constant drawn from $\text{Uniform}(0,1)$ at every iteration, where $t \in (0,0.1)$. The KL divergence is evaluated over 1 million samples for each experimental iteration for a given NFE, sampler, and noise schedule.

For the 1D dataset with 15 states, we employ commonly used noise schedules $\beta_t$ applied to a uniform rate matrix:
\paragraph{Linear.} The linear noise schedule simply assumes time homogeneity of the rate matrix:
\begin{align*}
    \beta_t &= 1 \\
    \int_0^t\beta(s)ds &= t
\end{align*}
\paragraph{Geometric.} The geometric noise schedule defines an exponentially growing rate with exponential base $b$ and constant factor $a$:
\begin{align*}
    \beta_t &= a(b^t)\ln(b) \\
    \int_0^t\beta(s)ds &= a(b^t)-a
\end{align*}
where we choose $a = 3$ and $b = 100$ for our experiments.
\paragraph{Loglinear.} The loglinear noise schedule is defined below, where we add an additional factor $a$ to accentuate the sensitivity of samplers to model approximation errors:
\begin{align*}
    \beta_t &= \frac{a(1-\epsilon)}{1-(1-\epsilon)t} \\
    \int_0^t\beta(s)ds &= -a\ln(1-(1-\epsilon)t)
\end{align*}
where we set $\epsilon = 10^{-3}$ and $a = 10$ for our experiments.

We provide a full comparison of all sampling configurations and experimental conditions in Figure~\ref{fig:noise_schedule}.

\begin{figure}[htbp]
    \centering
    \includegraphics[width=\linewidth]{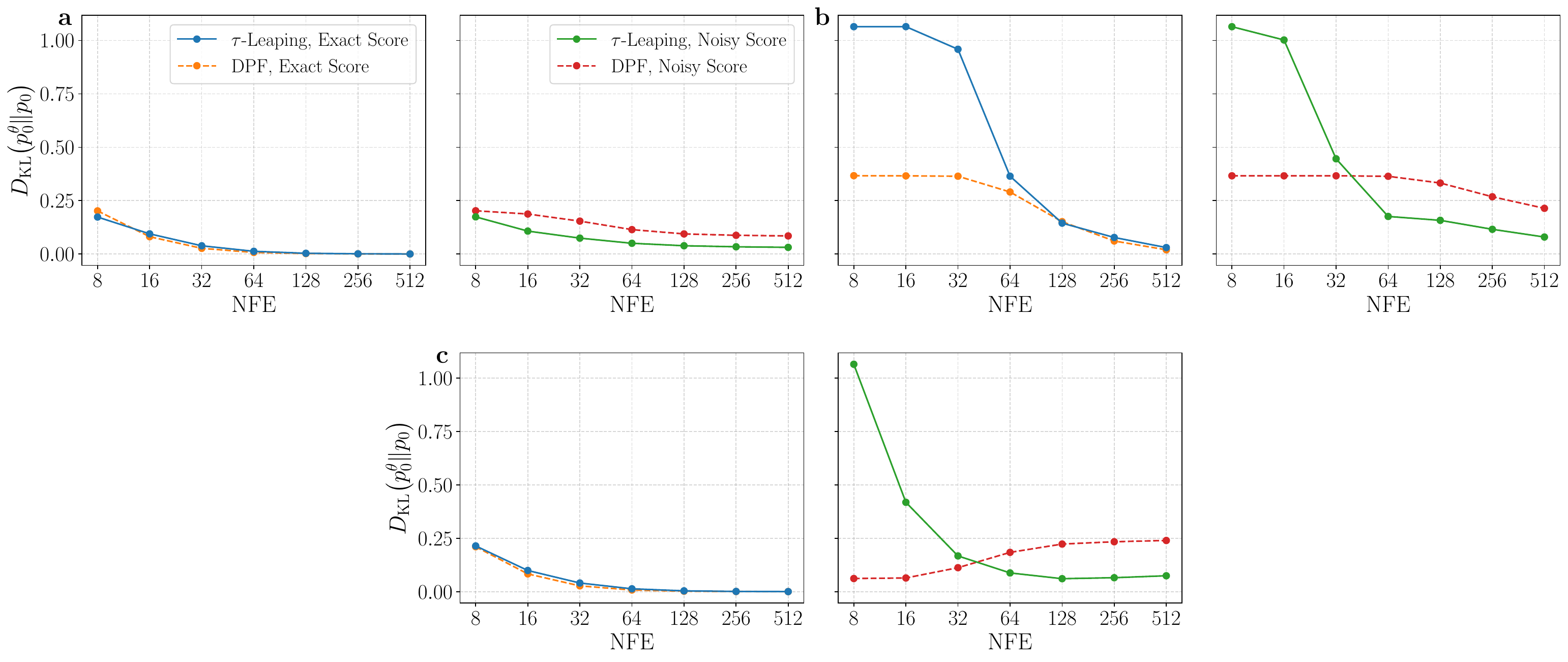}
    \caption{Empirical KL divergence between the ground-truth distribution $p_{\mathrm{data}}$ and generated distribution $\hat{p}_0$ for different noise schedules $\beta_t$: \textbf{(a)} linear, \textbf{(b)} geometric, and \textbf{(c)} loglinear. All experiments assume a uniform corruption process, where sampling performance is compared between using the exact discrete score function and a perturbed noisy score function.}
    \label{fig:noise_schedule}
\end{figure}
\subsubsection{Synthetic Dataset}
We use a mixture of Gaussians dataset containing 8 modes in 2 dimensions. We also present results on a separate synthetic dataset projected into 8 dimensions by randomly sampling a fixed projection matrix $P$ with independent standard Gaussian entries ($P_{ij} \sim \mathcal{N}(0,1)$). To preserve the scale of the dataset, the columns of $P$ are normalized to have unit length. Following \citet{DFOT,SDDM}, each coordinate is then quantized and is subsequently mapped onto 16-bit binary strings according to a Gray code. 

\label{section:details}
\subsubsection{Experimental Setup}
For all experiments, we use the Adam optimizer with a learning rate of 1e-4 and a batch size of 128. We parameterize models according to \citet{SDDM} and use a 3-layer MLP. An EMA was applied to all model weights with a constant of 0.9999. Instead of parameterizing the denoising posterior, the models learn the conditional distribution of each dimension $p_t(x^d\mid x^{\setminus d})$. By evaluating the model $D\times S$ times for a sequence of length $D$ with state space size $S$, we can evaluate the discrete score $s_{\theta}(x^d,y^d) \approx \frac{p_t(y^d\mid x^{\setminus d})}{p_t(x^d\mid x^{\setminus d})}$.

All models are trained using a single NVIDIA A6000 GPU for $100,000$ iterations. For absorbing state diffusion models, we also pass the model inputs through an additional embedding matrix. We train all models with a linear noise schedule $\beta_t = 1$ and perform inference
with the geometric schedule $\beta_t = ab^t\log(b)$ with $a = 3$ and $b = 100$,
following \citet{TLDR}.

\subsubsection{Sample Visualization}
After converting from binary strings back to floating point values, we measure sample quality through the Wasserstein-1 distance under the Euclidean distance metric. Performance comparisons are made between 10000 samples from both the ground-truth and learned distributions.

\begin{figure*}[htbp]
    \centering 
    \includegraphics[width=\linewidth]{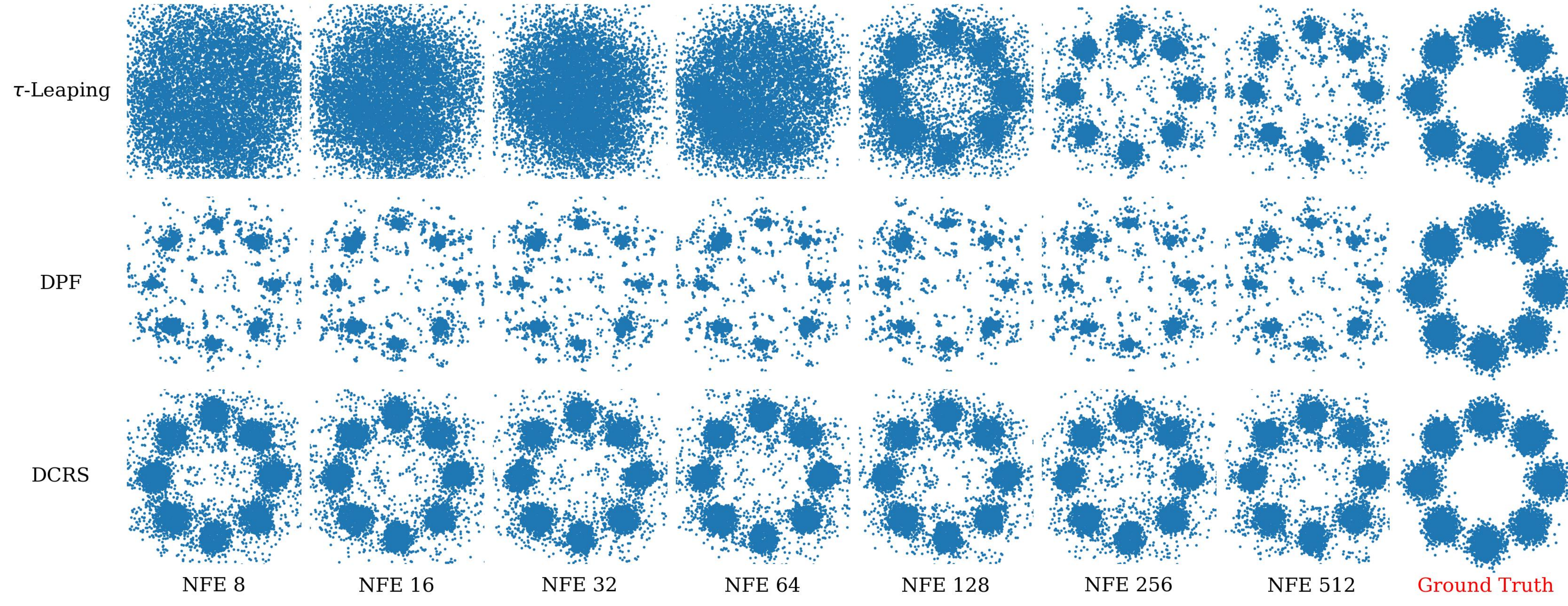}
    \caption{Plots of 10K generated samples as the number of discretization steps is increased for a toy mixture of Gaussians dataset in 8D from 2D
    experiments trained using a uniform corruption process. Whereas the $\tau$-leaping sampler converges slowly to the modes, the DPF sampler reaches the modes immediately but distributes mass unevenly. The best balance between sampling speed and quality is achieved with DCRS.}
    \label{fig:uniform8D_trajectory}
    \vspace{-1.5em}
\end{figure*}

\subsection{Images}
\subsubsection{Experimental Setup}
For the CIFAR10 and CelebA datasets, we use pretrained checkpoints provided by \citet{TLDR} and \citet{DFOT}. All experiments are conducted on $8$ NVIDIA A6000 GPUs. To enable comparisons between sampling algorithms, each sample is generated starting from a unique fixed seed.

\subsubsection{Effect of Multiple Restart Iterations}
We provide results using a time discretization of $8$ steps and a restart window of $[t_{\mathrm{min}},t_{\mathrm{max}}] = [0.7,0.8]$ with $3$ discretization steps. All discretization schemes are found using Algorithm \ref{alg:discretization_scheme} with $\rho = 1$ and $\rho = 7$, which corresponds to taking uniformly spaced steps and steps that are clustered near $t = 0$, respectively. We present the effect on FID over $10000$ generated samples as the number of restart iterations is changed. Our results in Figure~\ref{fig:restart_iterations} illustrate that the sampling quality tends to improve up until a certain number of iterations, after which the sample quality starts degrading due to the compounding of sampling errors. Furthermore, under non-uniform discretizations, such as when $\rho = 7$, the sample quality does not improve much with restart iterations, where multiple restart iterations tend to become harmful.

\begin{figure}[htbp]
  \centering
  \includegraphics[scale=0.5]{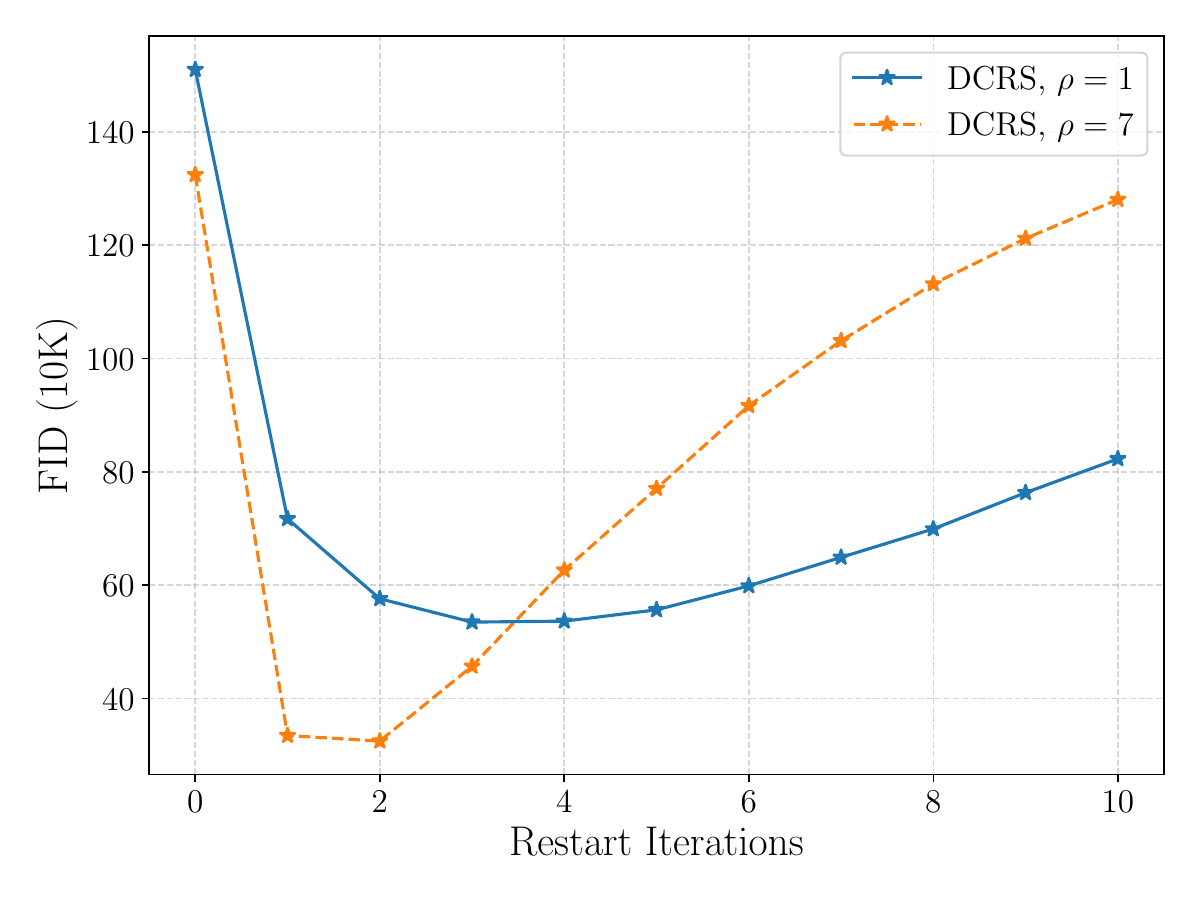}
  \caption{Effect of multiple restart iterations on sampling quality (FID) for a pre-trained model checkpoint trained on CIFAR10.}
  \label{fig:restart_iterations}
\end{figure}
\subsubsection{Failure of Higher-Order Solvers on DPF}
\label{section:failure}
We observe that applying a higher-order solver outside of the restart window tends to severely degrade quality, as shown in Figure~\ref{fig:trapezoid}. In particular, we observe that images tend to saturate in pixel intensities. We hypothesize that the added noise during the restart step tends to cancel out compounding errors under the DPF rate matrix, which are amplified under higher-order solvers. 

\begin{figure}[htbp]
    \centering
    \includegraphics[scale = 2]{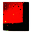}
    \caption{Example sample from running the trapezoidal solver (Algorithm \ref{alg:trapezoidal_step}) using the DPF rate matrix.}
    \label{fig:trapezoid}
\end{figure}
\newpage
\subsubsection{Restart Window Sensitivity}
Motivated by our theoretical results that the stochasticity of the transition matrix affects the sensitivity to sampling error, we offer a heuristic that estimates the average log-magnitude in the rate matrix across time steps during sampling. In Table \ref{tab:step_trajectory} we report empirical evidence that suggests that the optimal restart window we found, (0.7,0.8), is located approximately where the rate matrix is most sparse (0.87, 0.80), suggesting approximate agreement with our heuristic. We also include the rate matrix log-magnitudes for the default rate matrix, which do not experience the same dip in magnitude early on. 

\begin{table}[htbp]
\centering
\caption{Values at each sampling step for CIFAR-10 and CelebA, under stochastic and nearly deterministic sampling.}
\label{tab:step_trajectory}
\begin{tabular}{rrrrrr}
\toprule
 & & \multicolumn{2}{c}{CIFAR-10} & \multicolumn{2}{c}{CelebA} \\
\cmidrule(lr){3-4} \cmidrule(lr){5-6}
Step & $t$ & Stochastic & Nearly Det. & Stochastic & Nearly Det. \\
\midrule
0  & 1.00 & 17.70 & 10.06 & 20.47 & 12.83 \\
1  & 0.93 & 17.12 & 9.72  & 19.89 & 12.52 \\
2  & \textbf{0.87} & 16.67 & \textbf{9.59}  & 19.44 & \textbf{12.52} \\
3  & \textbf{0.80} & 16.28 & \textbf{9.91}  & 19.05 & \textbf{12.79} \\
4  & 0.74 & 15.92 & 10.27 & 18.70 & 12.96 \\
5  & 0.67 & 15.59 & 10.48 & 18.36 & 13.08 \\
6  & 0.60 & 15.27 & 10.59 & 18.04 & 13.18 \\
7  & 0.54 & 14.96 & 10.66 & 17.74 & 13.31 \\
8  & 0.47 & 14.98 & 10.73 & 17.75 & 13.44 \\
9  & 0.41 & 14.92 & 10.76 & 17.70 & 13.53 \\
10 & 0.34 & 14.65 & 10.75 & 17.42 & 13.56 \\
11 & 0.27 & 14.36 & 10.71 & 17.14 & 13.53 \\
12 & 0.21 & 14.07 & 10.67 & 16.84 & 13.49 \\
13 & 0.14 & 13.77 & 10.67 & 16.54 & 13.49 \\
14 & 0.08 & 13.47 & 10.83 & 16.25 & 13.64 \\
\bottomrule
\end{tabular}
\end{table}

 For the image domain experiments, we can find the optimal restart window via progressively refining a grid search, which we demonstrate results for in Table \ref{tab:cifar10_restart_intervals}:

\begin{table}[htbp]
\centering
\caption{CIFAR-10 FID by restart interval for $\rho=1$ and $\rho=7$, at three interval lengths.}
\label{tab:cifar10_restart_intervals}
\begin{tabular}{lrr}
\toprule
Restart interval & $\rho=1$ FID $\downarrow$ & $\rho=7$ FID $\downarrow$ \\
\midrule
\multicolumn{3}{l}{\textit{Interval length 0.25}} \\
\midrule
0.00--0.25 & 259.89 & 339.12 \\
0.25--0.50 & 37.31  & 77.68 \\
\textbf{0.50--0.75} & \textbf{24.52}  & \textbf{18.25} \\
0.75--1.00 & 34.65  & 45.66 \\
\midrule
\multicolumn{3}{l}{\textit{Interval length 0.2}} \\
\midrule
0.00--0.20 & 261.85 & 356.41 \\
0.20--0.40 & 45.72  & 82.68 \\
0.40--0.60 & 23.75  & 38.49 \\
\textbf{0.60--0.80} & \textbf{15.54}  & \textbf{11.00} \\
0.80--1.00 & 38.79  & 45.66 \\
\midrule
\multicolumn{3}{l}{\textit{Interval length 0.1}} \\
\midrule
0.00--0.10 & 249.24 & 359.36 \\
0.10--0.20 & 45.91  & 87.68 \\
0.20--0.30 & 56.63  & 76.75 \\
0.30--0.40 & 50.44  & 70.56 \\
0.40--0.50 & 44.26  & 51.54 \\
0.50--0.60 & 28.20  & 31.31 \\
0.60--0.70 & 22.91  & 23.88 \\
\textbf{0.70--0.80} & \textbf{16.03}  & \textbf{11.00} \\
0.80--0.90 & 31.08  & 27.12 \\
0.90--1.00 & 46.84  & 81.68 \\
\bottomrule
\end{tabular}
\end{table}
\newpage
\subsubsection{Ablation Studies}
\label{section:ablation}
For completeness, we provide additional results that illustrate the impact of higher-order solvers and discretization schemes applied to the $\tau$-leaping sampler in Figure~\ref{fig:ablation}. Our results indicate non-uniform discretization schemes where $\rho = 7$ can improve performance slightly in the low-NFE regime, but ultimately result in worse performance in the high-NFE regime. Furthermore, the higher-order solver seems to be somewhat ineffective, which we attribute to the stochasticity of the default reverse rate matrix. Like the $\tau$-leaping sampler, we expect higher-order solvers to converge faster on the DPF rate matrix, where sample trajectories exhibit fewer transitions. Ultimately, the results suggest that the discretization scheme and solver choices matter primarily when the stochasticity of the rate matrix is low enough.
\begin{figure}[htbp]
  \centering
  \includegraphics[scale=0.5]{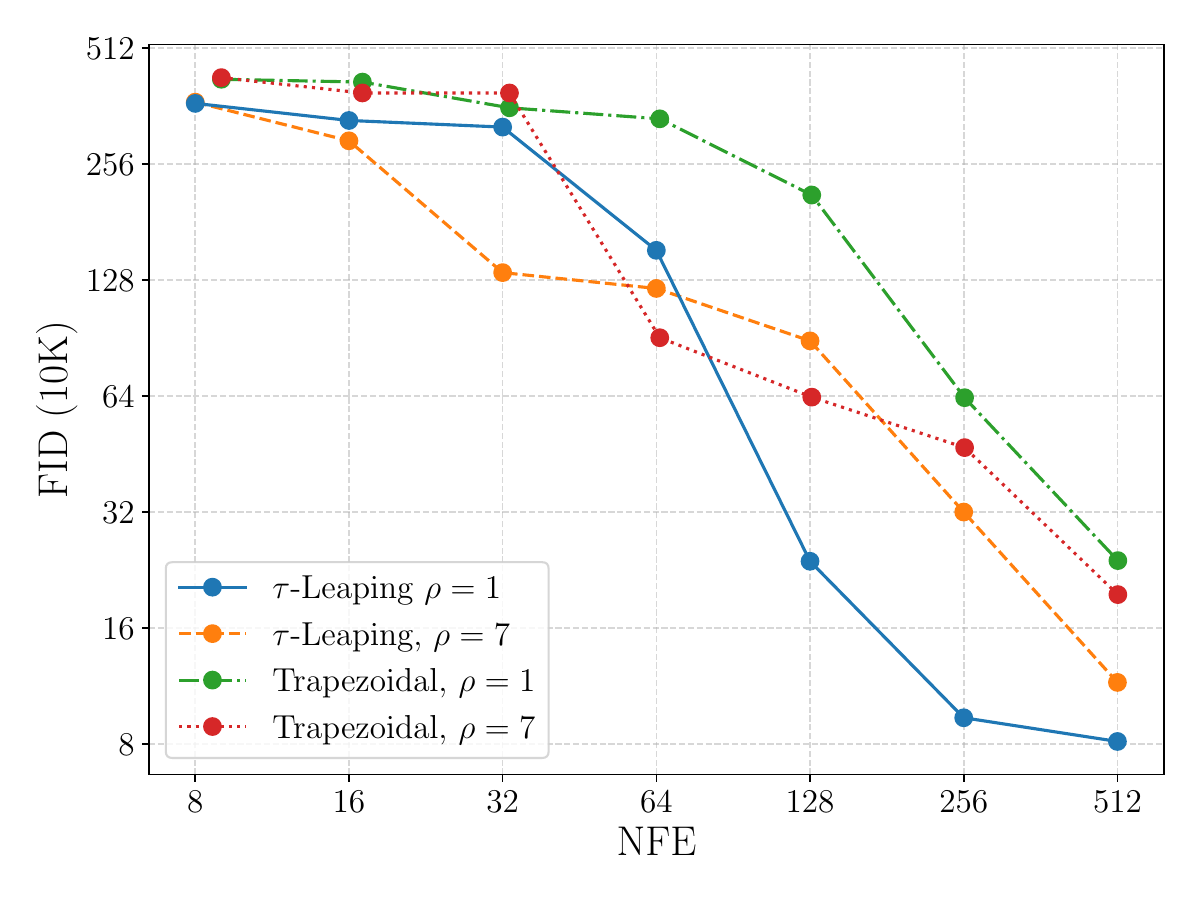}
  \caption{Effect of discretization scheme choices and higher-order solvers with the default $\tau$-leaping sampler on sampling quality (FID) for a pre-trained model checkpoint trained on CIFAR10.}
  \label{fig:ablation}
\end{figure}

\newpage
We also provide an additional predictor-corrector baseline, together with the DCRS algorithm and other sampling baselines for comparison for the CIFAR10 dataset. Specifically, we follow the experimental conditions in \cite{TLDR}, and use 3 predictor-corrector iterations. We use the same time discretization as DCRS ($\rho = 7$) and also provide results for the uniform discretization ($\rho = 1$). We also perform further ablations using the DPF rate matrix or the default rate matrix.
\begin{table}[htbp]
\centering
\caption{Predictor-corrector baseline comparison.}
\label{tab:fid_nfe}
\begin{tabular}{lrr}
\toprule
Method & Total Steps (NFE) & FID $\downarrow$ \\
\midrule
DPF & 8   & 150.99 \\
DPF & 16  & 51.66 \\
DPF & 32  & 33.66 \\
DPF & 64  & 39.43 \\
DPF & 128 & 43.29 \\
DPF & 256 & 45.22 \\
DPF & 512 & 46.47 \\
\midrule
$\tau$-Leaping & 8   & 366.74 \\
$\tau$-Leaping & 16  & 330.90 \\
$\tau$-Leaping & 32  & 316.35 \\
$\tau$-Leaping & 64  & 151.22 \\
$\tau$-Leaping & 128 & 21.73 \\
$\tau$-Leaping & 256 & 7.11 \\
$\tau$-Leaping & 512 & 6.09 \\
\midrule
DCRS & 12  & 31.33 \\
DCRS & 26  & 10.80 \\
DCRS & 54  & 11.37 \\
DCRS & 110 & 14.01 \\
DCRS & 222 & 15.32 \\
DCRS & 446 & 16.69 \\
\midrule
DPF+PC ($\rho=7$) & 16  & 268.73 \\
DPF+PC ($\rho=7$) & 36  & 62.13 \\
DPF+PC ($\rho=7$) & 70  & 43.82 \\
DPF+PC ($\rho=7$) & 144 & 35.09 \\
DPF+PC ($\rho=7$) & 289 & 33.28 \\
DPF+PC ($\rho=7$) & 576 & 33.73 \\
\midrule
DPF+PC ($\rho=1$) & 10  & 428.48 \\
DPF+PC ($\rho=1$) & 21  & 130.26 \\
DPF+PC ($\rho=1$) & 40  & 41.49 \\
DPF+PC ($\rho=1$) & 81  & 23.58 \\
DPF+PC ($\rho=1$) & 163 & 28.43 \\
DPF+PC ($\rho=1$) & 327 & 31.45 \\
DPF+PC ($\rho=1$) & 652 & 33.30 \\
\midrule
$\tau$-Leaping+PC & 10  & 410.36 \\
$\tau$-Leaping+PC & 21  & 322.03 \\
$\tau$-Leaping+PC & 40  & 292.98 \\
$\tau$-Leaping+PC & 81  & 134.31 \\
$\tau$-Leaping+PC & 163 & 20.23 \\
$\tau$-Leaping+PC & 327 & 7.11 \\
$\tau$-Leaping+PC & 652 & 6.06 \\ 
\bottomrule
\end{tabular}
\end{table}

Predictor-corrector sampling consistently underperforms DCRS except at high NFE. It improves quality at high NFE but performs poorly at low NFE, which can be attributed to elevated discretization error from larger rates. 
\newpage
\subsubsection{Sample Visualization}
We illustrate the difference in convergence rates for each dataset and choice of samplers in Figures \ref{fig:cifar10_trajectory} and \ref{fig:celeb128_trajectory} below. For each sampler, we use the samples generated by the same fixed seed using different total NFE. 
\begin{figure}[htbp]
    \centering
    % Subfigure 1
    \begin{subfigure}[t]{\linewidth}
        \centering
        \includegraphics[width=\linewidth]{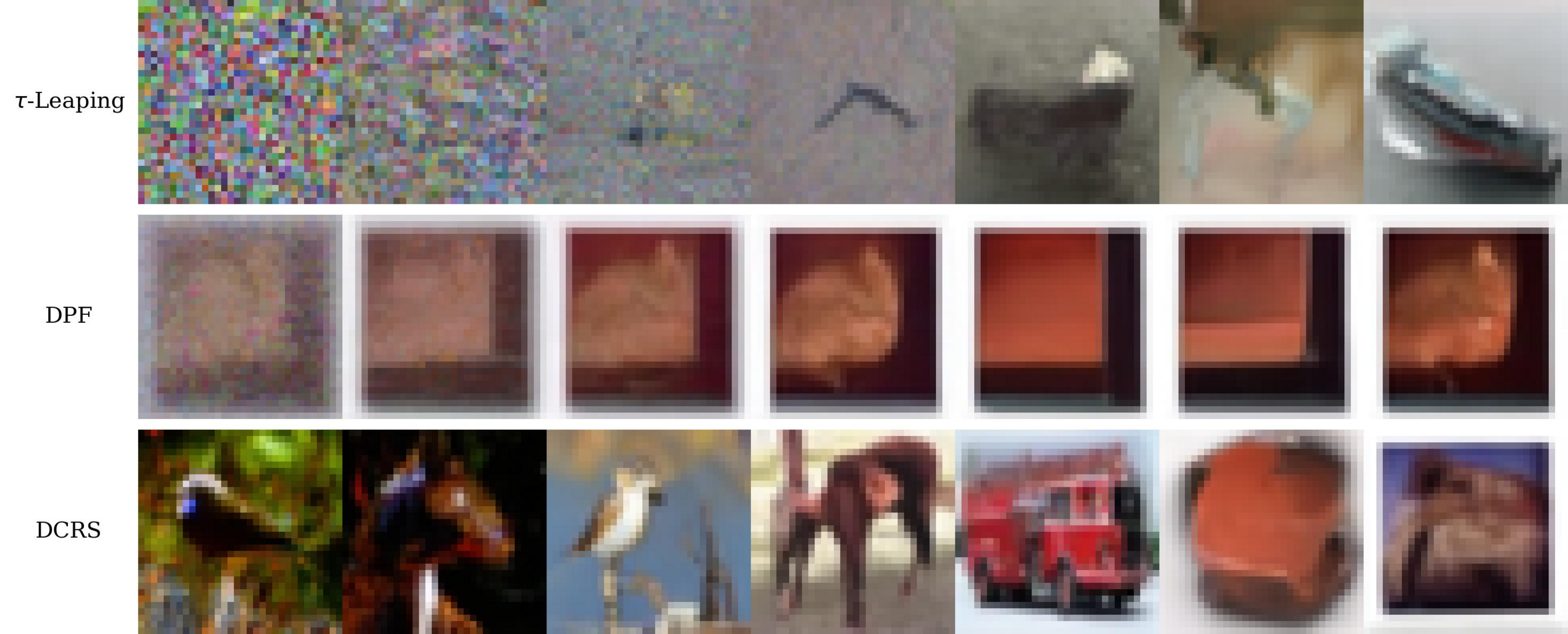}
        \caption{Generated samples using increasing NFE from left to right from a fixed seed for a pre-trained CIFAR10 checkpoint.}
        \label{fig:cifar10_trajectory}
    \end{subfigure}
    \hfill
    % Subfigure 2
    \begin{subfigure}[t]{\linewidth}
        \centering
        \includegraphics[width=\linewidth]{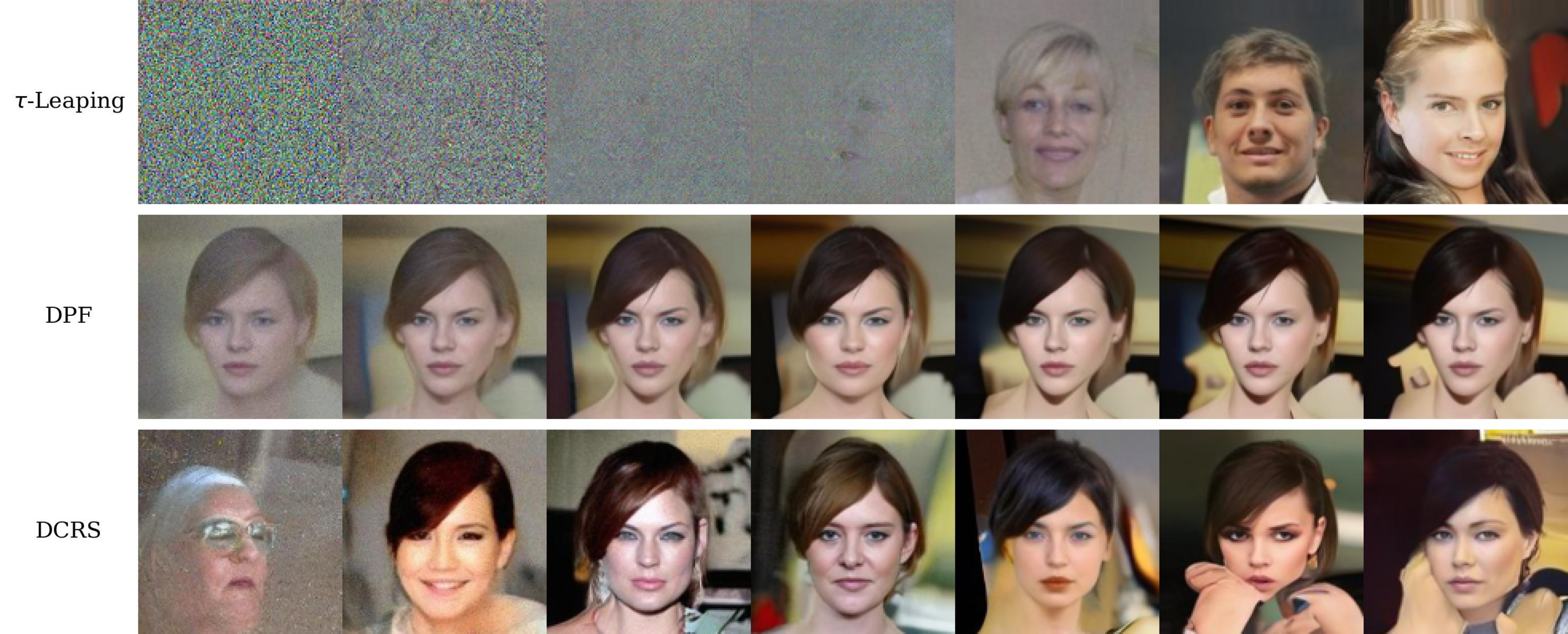}
        \caption{Generated samples using increasing NFE from left to right from a fixed seed for a pre-trained CelebA checkpoint.}
        \label{fig:celeb128_trajectory}
    \end{subfigure}

    \caption{Sample comparisons between the vanilla $\tau$-leaping, DPF, and DCRS sampler for large-scale image datasets.}
    \label{fig:image_trajectory}
\end{figure}
\newpage
\subsection{Language}
\label{section:lang}
\begin{table}[htbp]
\centering
\footnotesize
\begin{tabular}{gccccssss}
\toprule
\multicolumn{1}{l}{Method}
& \multicolumn{4}{c}{Generative PPL. $(\downarrow)$}
& \multicolumn{4}{c}{Entropy $(\uparrow)$} \\
\midrule
& \it{T=16} & \it{T=32} & \it{T=64} & \it{T=128}
& \it{T=16} & \it{T=32} & \it{T=64} & \it{T=128} \\
D3PM 
& 256.42 & 207.75 & 193.91 & 186.79
& 6.806 & 6.849 & 6.865 & 6.886 \\

DDIM$^\dagger$
& 253.60 & 213.57 & 205.21 & 190.66
& 6.831 & 6.859 & 6.891 & 6.898 \\

Euler 
& 272.97 & 225.41 & 197.87 & 190.10
& 6.921 & 6.922 & 6.890 & 6.902 \\

Euler DPF$^\dagger$
& 257.62 & 215.07 & 204.87 & 193.30
& 6.843 & 6.850 & 6.890 & 6.901 \\

\midrule
& \it{T=12} & \it{T=26} & \it{T=54} & \it{T=110} & \it{T=12} & \it{T=26} & \it{T=54} & \it{T=110} \\
Euler DCRS
& 298.44 & 245.06 & 331.60 & 333.73 &
6.704  & 6.823  & 7.080  & 7.115  \\
\bottomrule
\end{tabular}
\vspace{1em}
\caption{Generative perplexity and entropy evaluated for a uniform diffusion language model (UDLM) trained on LM1B over 1024 sequences of length 128. $^\dagger$ indicates a near-deterministic sampler that samples from the DPF CTMC.}
\vspace{-3em}
\label{tab:UDLM}
\end{table}

%\clearpage
%\input{checklist}
\end{document}